\RequirePackage{fix-cm}
\documentclass[smallextended,final]{svjour3}  

\def\makeheadbox{}

\smartqed  
\usepackage{graphicx}
\usepackage{amsfonts}
\usepackage{float}
\usepackage{algorithm}
\usepackage{algpseudocode}
\usepackage{hhline}
\usepackage{tabularx}  
\usepackage{ragged2e}  
\newcolumntype{Y}{>{\RaggedRight\arraybackslash}X} 
\usepackage{amsmath}

\usepackage{amsthm}

\DeclareMathOperator{\Tr}{Tr}
\newif\ifanon

\anonfalse

\allowdisplaybreaks

\journalname{}

\begin{document}

\title{Novel diffusion-derived distance measures for graphs}
\subtitle{Properties and efficient computation
\ifanon
\else
\thanks{This work was supported by the National Science Foundation NRT Award number 1633631, and also in part by the United States Air Force under Contract No. FA8750-14-C-0011 under the DARPA PPAML program; also by NIH grant R01 HD073179, and by the Leverhulme Trust, and by the hospitality of the Sainsbury Laboratory Cambridge University.}
\fi}

\titlerunning{Graph Diffusion Distance}

\ifanon
\else
\author{Cory B. Scott \and
Eric Mjolsness}
\authorrunning{C.B. Scott, E. Mjolsness} 
\institute{C.B. Scott \and E. Mjolsness \at University of California, Irvine, Dept. of Computer Science, Irvine, CA, USA \\
Cory Scott:  scottcb @ uci.edu \qquad Eric Mjolsness: emj @ uci.edu}
\fi

\date{Submitted September 9, 2019}

\maketitle

\begin{abstract}
We define a new family of similarity and distance measures on graphs, and explore their theoretical properties in comparison to conventional distance metrics. These measures are defined by the solution(s) to an optimization problem which attempts find a map minimizing the discrepancy between two graph Laplacian exponential matrices, under norm-preserving and sparsity constraints. Variants of the distance metric are introduced to consider such optimized maps under sparsity constraints as well as fixed time-scaling between the two Laplacians. The objective function of this optimization is multimodal and has discontinuous slope, and is hence difficult for univariate optimizers to solve. We demonstrate a novel procedure for efficiently calculating these optima for two of our distance measure variants. We present numerical experiments demonstrating that (a) upper bounds of our distance metrics can be used to distinguish between lineages of related graphs; (b) our procedure is faster at finding the required optima, by as much as a factor of $10^3$; and (c) the upper bounds satisfy the triangle inequality exactly under some assumptions and approximately under others. We also derive an upper bound for the distance between two graph products, in terms of the distance between the two pairs of factors. Additionally, we present several possible applications, including the construction of 
infinite ``graph limits'' by means of 
Cauchy sequences of graphs 
related to one another by 
our distance measure. 

\end{abstract}

\maketitle

\section{Introduction}
Structure comparison, as well as structure summarization, is a ubiquitous problem, appearing across multiple scientific disciplines.  In particular, many scientific problems (e.g. inference of molecular properties from structure, pattern matching in data point clouds and scientific images) may be reduced to the problem of inexact graph matching: given two graphs, compute a measure of similarity that gainfully captures structural correspondence between the two. Similarly, many algorithms for addressing multiple scales of dynamical behavior rely on methods for automatically coarsening some model architecture. 

In this work we present a graph distance metric, based on the Laplacian exponential kernel of a graph. This measure generalizes the work of Hammond et al. \cite{hammond2013graph} on graph diffusion distance for graphs of equal size; crucially, our distance measure allows for graphs of inequal size. We formulate the distance measure as the solution to an optimization problem dependent on a comparison of the two graph Laplacians. This problem is a nested optimization problem with the innermost layer consisting of multivariate optimization subject to matrix constraints (e.g. orthogonality). To compute this dissimilarity score efficiently, we also develop and demonstrate the lower computational cost of an algorithm which calculates upper bounds on the distance. This algorithm produces a prolongation/restriction operator, $P$, which produces an optimally coarsened version of the Laplacian matrix of a graph. Prolongation/restriction operators produced via the method in this paper have previously been applied to accelerate the training of machine learning algorithms in \cite{scott2018multilevel}.
\vspace{-5pt}
\subsection{Prior Work}
\label{subsec:prior_work}

Quantitative measures of similarity or dissimilarity between graphs
have been studied for decades owing to their relevance for
problems in pattern recognition including structure-based
recognition of extended and compound objects in computer vision,
prediction of chemical similarity based on
shared molecular structure, and many other domains.
Related problems arise in quantitative modeling, for example
in meshed discretizations of partial differential equations
and more recently in trainable statistical models of data 
that feature graph-like models of connectivity such
as Bayes Networks, Markov Random Fields, 
and artificial neural networks. A core problem is to
define and compute how ``similar'' two graphs are
in a way that is invariant to a permutation of the 
the vertices of either graph, so that the answer doesn't depend
on an arbitrary numbering of the vertices.
On the other hand unlike an arbitrary numbering,
problem-derived semantic {\it labels} on graph vertices
may express real aspects of a problem domain
and may be fair game for detecting graph similarity.
The most difficult case occurs when such labels are absent,
for example in an unstructured mesh, as we shall assume.

We mention just a few prior works to give an overview of the development of graph distance measures over time, paying special attention to those which share theoretical or algorithmic characteristics with the measure we introduce.
Our mathematical distinctions concern the
following properties:
(a) Does the distance measure require an inner optimization loop?
If so is it mainly a discrete or continuous optimization formulation?
(b) Does the distance measure calculation naturally yield
some kind of explicit {\it map} 
from real-valued functions on vertices of one graph 
to functions on vertices of the other? (A map from vertices to vertices would be a special case.)
If we use the term ``graph signal'' to mean a function $f: V(G_1) \rightarrow S$ which identifies each vertex of a graph $G_1$ with some state $s \in S$, then a map-explicit graph distance is one
whose calculation yields a second function $g: V(G_2) \rightarrow V(G_1)$, with the composite function $f \circ g: V(G_2) \rightarrow S$.
(c) Is the distance metric definable on the spectrum of
the graph alone, without regard to other data 
from the same graph? 
The ``spectrum'' of a graph is a graph invariant calculated as the eigenvalues of a matrix related to the adjacency matrix of the graph. Depending on context, the spectrum can refer to eigenvalues of the adjacency matrix, graph Laplacian, or normalized graph Laplacian of a graph. We will usually take the underlying matrix to be the graph Laplacian, defined in detail in Section \ref{subsec:defns}.
Alternatively, does it take into account more detailed
``structural'' aspects of the graph?
This categorization (structural vs. spectral) is similar to that introduced in \cite{donnat2018tracking}. 

\vspace{-30pt}
\subsubsection{Graph-Edit Distance (structural, explicit, disc-opt): } 
The graph edit distance measures the total cost of converting one graph into another with a sequence of local edit moves, with each type of move (vertex deletion/addition, edge deletion/addition, edge division/contraction) incurring a specified cost. Costs are chosen to suit the graph analysis problem at hand; determining a cost assignment which makes the edit distance most instructive for a certain set of graphs is an active area of research. The distance measure is then the sum of these costs over an optimal sequence of edits, which must be found using some optimization algorithm i.e. a shortest-path algorithm (the best choice of algorithm may vary, depending on how the costs are chosen). The sequence of edits may or may not (depending on the exact set of allowable edit moves) be adaptable into an explicit map between vertex-sets.
Classic pattern recognition literature includes:
\cite{eshera1984graph} \cite{eshera1986image} 
\cite{gao2010survey} \cite{sanfeliu1983distance} .

\vspace{-35pt}
\subsubsection{Quadratic Point- and Graph- Matching (structural, explicit, cont-opt)}
    Other work focuses on the construction of a point-to-point correspondence between the vertices of two graphs. Gold et. al. \cite{gold1995new} define the dissimilarity between two unlabelled weighted graphs with adjacency matrices $A^{(1)}$ and $A^{(2)}$ as the solution to the following optimization problem (for real-valued $M = [m_{ij}]$:

    \begin{equation}
        \label{eqn:gold_defn}
        \begin{array}{ll@{}ll}
        \text{minimize}  & \displaystyle \sum\limits_{j=1}^{n_2}\sum\limits_{k=1}^{n_1} {\left( 
        \sum\limits_{l=1}^{n_2} A^{(1)}_{jl} m_{lk} - \sum\limits_{p=1}^{n_1} m_{j p} A^{(2)}_{pk}
        
        \right)}^2 &= {\left|\left|  A^{(1)} M - M  A^{(2)} \right| \right|}_F^2\\
        \text{subject to}& \displaystyle\sum\limits_{i=1}^{n_2}  m_{ij} = 1,  &j=1 \ldots n_1\\
               & \displaystyle\sum\limits_{j=1}^{n_1}  m_{ij} = 1,  &i=1 \ldots n_2\\
         & m_{ij} \geq 0 & i = 1 \ldots n_2\\
         & &j = 1 \ldots n_1 \\
        \end{array}
    \end{equation}
    where ${\left|\left|  \cdot \right| \right|}_F^2$ is the squared Frobenius norm. This problem is similar in structure to the optimization considered in Section \ref{sec:alg}: a key difference being that Gold et al. consider optimization over real-valued matchings between graph vertices, whereas we consider 0-1 valued matchings between the eigenvalues of the graph Laplacians. In \cite{gold1995learning} and \cite{rangarajan1996novel} the authors present computational methods for computing this optimum, and demonstrate applications of this distance measure to various machine learning tasks such as 2D and 3D point matching, as well as graph clustering. Gold et al. also introduce the \emph{softassign}, a method for performing combinatorial optimization with both row and column constraints, similar to those we consider. 
   
\subsubsection{Cut-Distance of Graphs (structural, implicit, disc-opt):}

Lov\'asz \cite{lovasz2012large} defines the \emph{cut-distance} of a pair of graphs as follows: Let the $\Box$-norm of a matrix $B$ be given by:
    \begin{align}
        {\left| \left| B \right| \right|}_\Box = \frac{1}{n^2} \max_ {S, T \subseteq 1 \ldots n} \left| \sum_{i \in S, j \in T} B_{ij} \right|
    \end{align}

    Given two labelled graphs $G_1$, $G_2$, on the same set of vertices, and their adjacency matrices $A_1$ and $A_2$, the cut-distance $d_\text{cut}(G_1, G_2)$ is then given by
    \begin{align}
        \label{defn:cut_dist}
        D_\text{cut}(G_1, G_2) = {\left| \left| A_1 - A_2 \right| \right|}_\Box
    \end{align}
    (for more details, see \cite{lovasz2012large}). Computing this distance requires combinatorial optimization (over all vertex subsets of $G_1,G_2$) but this optimization does not result in an explicit map between $G_1$ and $G_2$. 

\subsubsection{Wasserstein Earth-Mover Distance (spectral, implicit, disc-opt):}
One common metric between graph spectra is the Wasserstein Earth Mover Distance. Most generally, this distance measures the cost of transforming one probability density function into another by moving mass under the curve. If we consider the eigenvalues of a (possibly weighted) graph as point masses, then the EMD measures the distance between the two spectra as the solution to a transport problem (transporting one set of points to the other, subject to constraints e.g. a limit on total distance travelled or a limit on the number of `agents' moving points). 
The EMD has been used in the past in various graph clustering and pattern recognition contexts; see \cite{gu2015spectral}. In the above categorization, this is an optimization-based spectral distance measure, but is implicit, since it does not produce a map from vertices of $G_1$ to those of $G_2$ (informally, this is because the EMD is not translating one set of eigenvalues into the other, but instead transforming their respective histograms). Recent work applying the EMD to graph classification includes \cite{dodonova2016kernel} and  \cite{mcgregor2013sketching}.

\subsubsection{Diffusion Distance due to Hammond et al. \cite{hammond2013graph} (spectral, implicit, non-opt):} We discuss this
recent distance metric more thoroughly below. 
This distance measures the difference between two graphs as the maximum discrepancy between probability distributions which represent single-particle diffusion beginning from each of the nodes of $G_1$ and $G_2$. This distance is computed by comparing the eigenvalues of the heat kernels of the two graphs. The optimization involved in calculating this distance is 
a simple unimodal optimization
over a single scalar, $t$, representing the passage of time for the diffusion process on the two graphs; hence we do not count this among the ``optimization based'' methods we consider. 

\subsubsection{Novel Diffusion-Derived Measures (spectral, explicit, cont-opt):}
    In this work, we introduce a family of related graph distance measures which compare two graphs in terms of similarity of a set of probability distributions describing single-particle diffusion on each graph. For two graphs $G_1$ and $G_2$ with respective Laplacians $L(G_1)$ and $L(G_2)$, the matrices $e^{t L(G_1)}$ and $e^{t L(G_2)}$ are called the \emph{Laplacian Exponential Kernels} of $G_1$ and $G_2$ ($t$ is a scalar representing the passage of time). The column vectors of these matrices describe the probability distribution of a single-particle diffusion process starting from each vertex, after $t$ time has passed. The norm of the difference of these two kernels thus describes how different these two graphs are, from the perspective of single-particle diffusion, at time $t$. Since these distributions are identical at very-early and very late times $t$ (we formalize this notion in Section \ref{subsec:diff_defns}), a natural way to define a graph distance is to take the supremum over all $t$. When the two graphs are the same size, so are the two kernels, which may therefore be directly compared. This case is the case considered by Hammond et al. \cite{hammond2013graph}. However, to compare two graphs of different sizes, we need a mapping between the column vectors of $e^{t L(G_1)}$ and $e^{t L(G_2)}$. Optimization over a suitably constrained prolongation/restriction operator between the graph Laplacians of the two graphs is a permutation-invariant way to compare the behavior of a diffusion process on each. The prolongation map $P$ thus calculated may then be used to map signals (by which we mean values associated with vertices or edges of a graph) on $G_1$ to the space of signals on $G_2$ (and vice versa). In \cite{scott2018multilevel} we implicitly consider the weights of an artificial neural network model to be graph signals, and use these operators to train a hierarchy of linked neural network models. However, in that work we do not address efficient calculation of this distance or the associated operators, a major focus of this paper. In this work we also consider a time conversion factor between diffusion on graphs of unequal size, and consider the effect of limiting this optimization to sparse maps between the two graphs (again, our case reduces to Hammond when the graphs in question are the same size, dense $P$ and $R$ matrices are allowed, and our time-scaling parameter is set to 1).
    
        In this work, we present an algorithm for computing the type of nested optimization given in our definition of distance (Equations \ref{defn:exp_defn} and \ref{defn:linear_defn}). The innermost loop of our distance measure optimization consists of a Linear Assignment Problem (LAP, defined below) where the entries of the cost matrix have a nonlinear dependence on some external variable. Our algorithm greatly reduces both the count and size of calls to the external LAP solver. We use this algorithm to compute an upper bound on our 
        distance measure, but it could also be useful in other similar nested optimization contexts: specifically, nested optimization where the inner loop consists of a linear assignment problem whose costs depend quadratically on the parameter in the outermost loop.

\vspace{-30pt}
\subsection{Background}

The ideal for a quantitative measure of similarity
or distance on some set $S$ is usually taken to be
a distance {\it metric} $d: S \times S \mapsto {\mathbb R}$
satisfying for all $x,y,z \in S$:
\begin{itemize}
    \item Non-negativity: $d(x,y) \geq 0$
    \item Identity: $d(x,y) = 0 \iff x=y$
    \item Symmetry: $d(x,y) = d(y,x)$
    \item Triangle inequality: $d(x,z) \leq d(x,y) + d(y,z)$
\end{itemize}
Then $(S,d)$ is a {\it metric space}.
Euclidean distance on ${\mathbb R}^d$ 
and geodesic distance on manifolds 
satisfy these axioms.
They can be used to define algorithms
that generalize from ${\mathbb R}^d$ to other spaces.
A variety of weakenings of these axioms are required
in many applications, by dropping some axioms and/or 
weakening others.
For example
if $S$ is a set of nonempty sets of a metric space $S_0$,
one can define the ``Hausdorff distance'' on $S$ which 
is an {\it extended pseudometric} that
obeys the triangle inequality but
not the Identity axiom and that can take values including
$+\infty$. As another example, any measure measure of distance on graphs which is purely spectral (in the taxonomy of Section \ref{subsec:prior_work}) cannot distinguish between graphs which have identical spectra. We discuss this in more detail in  Section \ref{subsec:dist_variants}.

Additional properties of distance metrics that generalize Euclidean distance 
may pertain to metric spaces related by Cartesian product, for example,
by summing the squares of the distance metrics
on the factor spaces.
We will consider an analog of this property in 
Section~\ref{graph_product_bound}.

\subsection{Definitions}
\label{subsec:defns}
\begin{itemize}
    \item Graph Laplacian: For an undirected graph $G$ with adjacency matrix $A$ and vertex degrees $d_1, d_2 \ldots d_n$, we define the Laplacian of the graph as \begin{align} L(G) = A - \textit{diag}(\{d_1, d_2 \ldots d_n\}) = A - \textit{diag}(\mathbf{1} \cdot A) = A(G) - D(G)\end{align} $L(G)$ is sometimes instead defined as $D(G)-A(G)$; we take this sign convention for $L(G)$ because it agrees with the standard continuum Laplacian operator, $\Delta$, of a multivariate function $f$: $\Delta f = \sum_{i=1}^n \frac{\delta^2 f}{\delta x_i^2}$.
    \item The squared Frobenius norm, ${\left| \left| A \right| \right|}^2_F$ of a matrix $A$ is given by the sum of squares of matrix entries. This can equivalently be written as $\Tr[A^T A]$.
    \item Linear Assignment Problem (LAP): we take the usual definition of the Linear Assignment Problem (see \cite{burkard1999linear}, \cite{burkard2009assignment}): we have two lists of items $S$ and $R$ (sometimes referred to as ``workers'' and ``jobs''), and a cost function $c: S \times R \rightarrow \mathbb{R}$ which maps pairs of elements from $S$ and $R$ to an associated cost value. This can be written as a linear program for real-valued $x_{ij}$ as follows:
    \begin{equation}
        \label{eqn:lap_defn}
        \begin{array}{ll@{}ll}
        \text{minimize}  & \displaystyle\sum\limits_{i=1}^{m}\sum\limits_{j=1}^{n} c(s_i,r_j)x_{ij} &\\
        \text{subject to}& \displaystyle\sum\limits_{i=1}^{m}  x_{ij} \leq 1,  &j=1 \ldots n\\
               & \displaystyle\sum\limits_{j=1}^{n}  x_{ij} \leq 1,  &i=1 \ldots m\\
         & x_{ij} \geq 0 & i = 1 \ldots m, j = 1 \ldots n \\
        \end{array}
    \end{equation}
    Generally, ``Linear Assignment Problem'' refers to the square version of the problem where $|S| = |R| = n$, and the objective is to allocate the $n$ jobs to $n$ workers such that each worker has exactly one job and vice versa. The case where there are more workers than jobs, or vice versa, is referred to as a Rectangular LAP or RLAP.
    In practice, the conceptually simplest method for solving an RLAP is to convert it to a LAP by \emph{augmenting} the cost matrix with several columns (rows) of zeros. In this case, solving the RLAP is equivalent to solving a LAP with size $max(n,m)$. Other computational shortcuts exist; see \cite{bijsterbosch2010solving} for details. Since the code we use to solve RLAPs takes the augmented cost matrix approach, we do not consider other methods in this paper.  
    \item Matching: we refer to a 0-1 matrix $M$ which is the solution of a particular LAP as a ``matching''. We may refer to the ``pairs'' or ``points'' of a matching, by which we mean the pairs of indices $(i,j)$ with $M_{ij} = 1$. We may also say in this case that $M$ ``assigns'' $i$ to $j$. 
    \item Graph Lineage: we use the same definition of graph lineage as in \cite{scott2018multilevel}, namely: 
    A graph lineage is a sequence of graphs, indexed by $ l \in \mathbb{N} = 0, 1, 2, 3 \ldots$, satisfying the following:
    \begin{itemize}
        \item $G_0$ is the graph with one vertex and one self-loop, and;
        \item Successive members of the lineage grow roughly exponentially - that is, there exists some base $b$ such that the growth rate as a function of level number $l$ is $O(b^{l^{1+\epsilon}})$, for all $\epsilon > 0$.
    \end{itemize}
    \item Box product $(\Box)$ of graphs: For two graphs $G$ and $H$ with vertex sets $V(G) = \{ g_1, g_2 \ldots g_n \}$ and $V(H) = \{h_1, h_2 \ldots h_m \}$, we say the product graph $G \Box H$ is the graph with vertex set $V(G \Box H) = V(G) \times V(H)$ and an adjacency relationship defined by: $(g_1, h_1) \sim (g_2, h_2)$ in $G \Box H$ if and only if $g_1 \sim g_2$ in $G$ and $h_1 = h_2$, or $g_1 = g_2$ and $h_1 \sim h_2$ in $H$. Note that the adjacency matrix of this relationship may be represented by the following identity:
    \begin{equation}
        \begin{split}
            A(G \Box H) = A(G) \otimes I_m + I_n \otimes A(H)
        \end{split}
    \end{equation}
    where $\otimes$ is the Kronecker Product of matrices (See \cite{hogben2006handbook}, Section 11.4).
\end{itemize}

\vspace{-23pt}

\section{Graph Diffusion Distance Definitions}

\subsection{Diffusion Distance Definition}
\label{subsec:diff_defns}
    We generalize the diffusion distance defined by Hammond et al. \cite{hammond2013graph}. This distortion measure between two graphs $G_1$ and $G_2$, of the same size, was defined as:
    \begin{align}
        D_{\text{Hammond}} (G_1, G_2) = \sup_{t} { \left| \left| e^{t L_1} -  e^{t L_2} \right| \right| }_F^2
    \end{align}
    where $L_i$ represents the graph Laplacian of $G_i$. A typical example of this distance, as $t$ is varied, can be seen in Figure \ref{fig:hammond_distance}.
    \begin{figure}
        \centering
        \includegraphics[width=.8\linewidth]{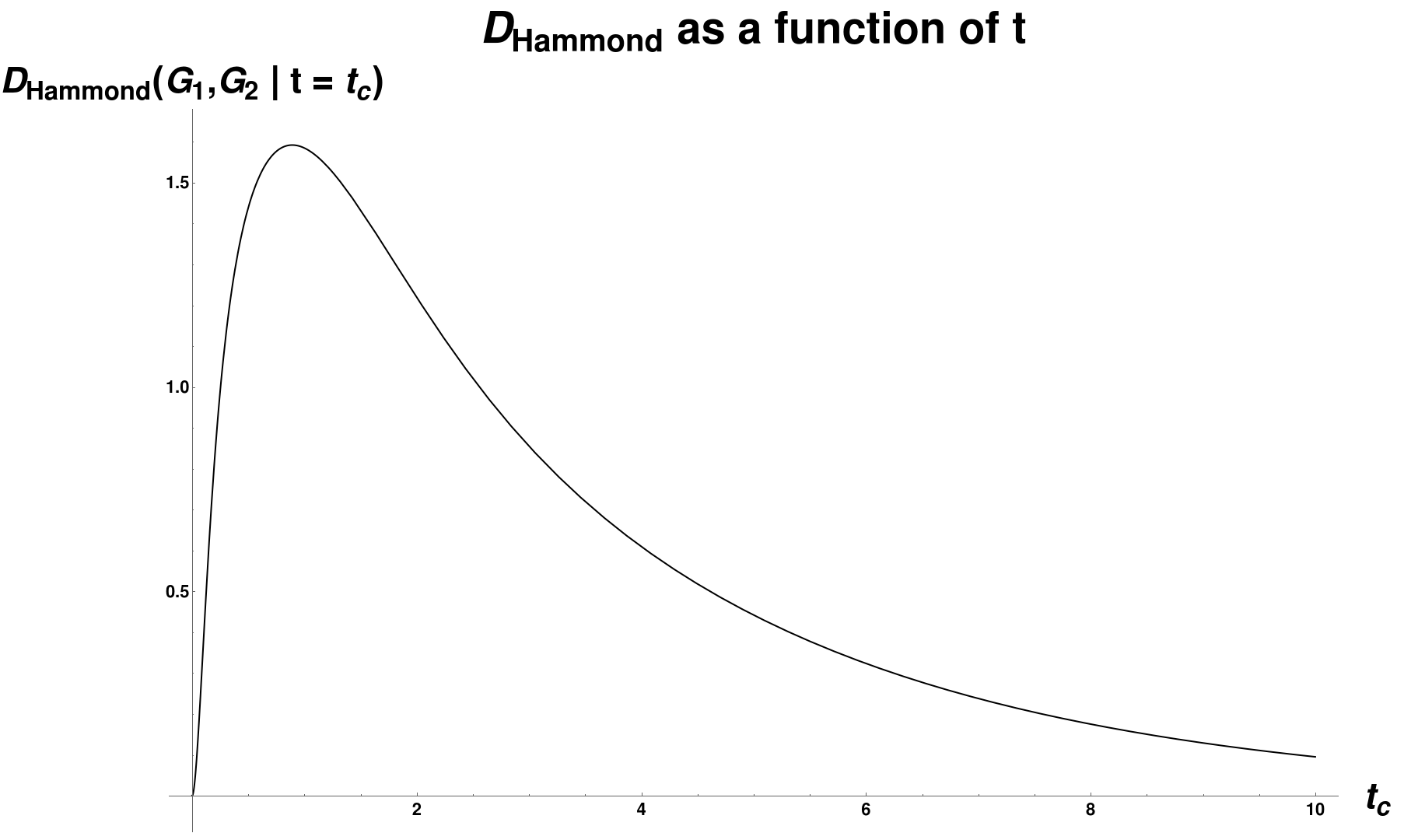}
        \caption{Plot of graph distance between two graphs of the same size (a $4 \times 4$ grid and a path graph on 16 vertices), following the definition of Hammond et al. \cite{hammond2013graph}, as the parameter $t$ is varied. Note 1) there is single global maximum of $1.592$ at $t= 0.891$ and 2) similarity to the plot of our distance measure as a function of $t$ given in Figure \ref{fig:distance_example_fig}(a).}
        \label{fig:hammond_distance}
    \end{figure}
    This may be interpreted as measuring the maximum divergence, as a function of $t$, between diffusion processes starting from each vertex of each graph, as measured by the squared Euclidean distance between the column vectors of $e^{t L_i}$. Each column $v_j$ of $e^{t L_i}$  (which is called the Laplacian Exponential Kernel) describes a probability distribution of visits (by a random walk of duration $t$, with node transition probabilities given by the columns of $e^L$) to the vertices of $G_i$, starting at vertex $j$. This distance metric is then measuring the difference between the two graphs by comparing these probability distributions; the motivation between taking the supremum over all $t$ is that the value of the objective function at the maximum is the most these two distributions can diverge. For further intuition about why the peak is the most natural place to take as the distance, rather than some other arbitrary time, note that at very early times and very late times, the probability distribution of vertex visits is agnostic to graph structure: at early times no diffusion has had a chance to take place, while at very late times the distribution of vertex-visits converges to the uniform distribution for each connected component of the graph. Hence we are most interested in a regime of $t$-values in between these extremes, where differences in $G_1$ and $G_2$ are apparent in their differing probability distributions. 
    
    Our contribution generalizes this measure to allow for graphs of differing size. We add two variables to this optimization: a \emph{prolongation} operator, $P$ (represented as a rectangular matrix), and a time-scaling factor, $\alpha$. The dissimilarity between two graphs $G_1$ and $G_2$ (with Laplacians $L_i = L(G_i)$) is then defined as:
    \begin{align}
        \label{defn:exp_defn}
            D^2(G_1,G_2) &= \sup_{t>0} \inf_{\alpha > 0} \inf_{P | \mathcal{C}(P)} {\left| \left| P e^{\frac{t}{\alpha} L_1} - e^{\alpha t L_2} P \right| \right|}^2_F 
    \end{align}
    where $\mathcal{C}(P)$ represents some set of constraints on the matrix $P$. It will be convenient for later calculations to introduce and assume the concept of  \emph{transitive constraints} - by which we mean that for any constraint $\mathcal{C}$, satisfaction of $\mathcal{C}$ by $P_1$ and $P_2$ implies satisfaction of $\mathcal{C}$ by their product $P_1 P_2$ (when such a product is defined). Some (non-exclusive) examples of transitive constraints include orthogonality, particular forms of sparsity, and their conjunctions.

    The simplest transitive constraint we will consider
    is that $P$ should be orthogonal.
   Intuitively, an orthogonal $P$ represents a norm-preserving map between nodes of $G_1$ and nodes of $G_2$, so we are measuring how well diffusion on $G_1$ approximates diffusion on $G_2$, as projected by $P$.   Note that since in general $P$ is a rectangular matrix it is not necessarily true that $P P^T = I$.  We assume that $\left| G_1\right| = n_1 \leq n_2 =  \left| G_2 \right|$; if not, the order of the operands is switched, so that $P$ is always at least as wide as it is tall. 
   We also briefly consider a sparsity constraint in
    section~\ref{sparsity_section} below.
    Since sparsity is more difficult to treat numerically,
    our default constraint will be orthogonality alone.
    Other constraints could include bandedness and other structural constraints
    (see Section \ref{sec:futurework}). 
    We also note that because $L$ is finite-dimensional, the exponential map is continuous and therefore
    we can swap the order of optimization over $t$ and $\alpha$. The optimization procedure outlined in this paper optimizes these variables in the order presented above (namely: an outermost loop of maximization over $t$, a middle loop of minimization over $\alpha$ given $t$, and an innermost loop of minimization over $P$ given $t$ and $\alpha$). 
    
    The other additional parameter, $\alpha$, controls dilation between the passage of time in the two graphs, to account for different scales. Again, the intuition is that we are interested in the difference between structural properties of the graph (from the point of view of single-particle diffusion) independent of the absolute number of nodes in the graph. As an example, diffusion on an $n \times n$ grid is a reasonably accurate approximation of more rapid diffusion on a $2n \times 2n$ grid, especially when $n$ is very large. In our discussion of variants of this dissimilarity score, we will use the notation $D^2(G_1, G_2 | x =  c)$ to mean restrictions of any of our distortion measure equations where variable $x$ is held to a constant value $c$; In cases where it is clear from context which variable is held to a fixed value $c$, we will write $D^2(G_1, G_2 | c)$
    
    At very early times the second and higher-order terms of the Taylor Series expansion of the matrix exponential function vanish, and so $e^{t L} \approx I + t L$.  This motivates the \emph{early-time} or ``linear'' version of this distance, $\tilde{D}$: 
    \begin{align}
    \label{defn:linear_defn}
        \tilde{D}^2(G_1,G_2) &= \inf_{\alpha > 0} \inf_{P | \mathcal{C}(P)} {\left| \left| \frac{1}{\alpha} P L_1 - \alpha L_2 P \right| \right|}^2_F \\
        &\approx \frac{1}{t^2} \left( \inf_{\alpha > 0} \inf_{P | \mathcal{C}(P)} {\left| \left| P e^{\frac{t}{\alpha} L_1} - e^{\alpha t L_2} P \right| \right|}^2_F \right)
    \end{align}
    
    (Note that the identity matrices cancel). The outermost optimization (maximization over $t$) is removed for this version of the distance, as the factor of $t$ can be pulled out: 
    
    \begin{align}
         {\left| \left| \frac{t}{\alpha} P L_1 - \alpha t L_2 P \right| \right|}^2_F 
         = 
         t^2 {\left| \left| \frac{1}{\alpha} P L_1 - \alpha L_2 P \right| \right|}^2_F 
    \end{align}
    
    For the exponential version of the dissimilarity score, we note briefly that the supremum over $t$ of our objective function must exist, since for any $G_1, G_2$:
    \begin{align}
        \label{eqn:D_limit_1}
        D^2(G_1, G_2) \leq D^2 \left(G_1, G_2 \left| \alpha = 1, P= \left[ \begin{array}{c}
                 I \\
                 0 
            \end{array}\right] \right. \right)
    \end{align}
    In other words, the infimum over all $P$ and $\alpha$ is bounded above by any particular choice of values for these variables. Since 
    \begin{align}
        \label{eqn:D_limit_2}
        D^2 \left(G_1, G_2 \left| t = 0, \alpha = 1, P =  \left[ \begin{array}{c}
                 I \\
                 0 
            \end{array}\right] \right. \right) = 0,
        \intertext{and}    
        \lim_{t_c \rightarrow \infty} D^2 \left(G_1, G_2 \left| t_c, \alpha = 1, P= \left[ \begin{array}{c}
                 I \\
                 0 
            \end{array}\right] \right. \right) = 0
    \end{align}
    this upper bound must have a supremum (possibly 0) at some $t^* \in [ 0, \infty)$. Then 
    \begin{align}
        D^2 \left(G_1, G_2 \left| t^*, \alpha = 1, P = \left[ \begin{array}{c}
                 I \\
                 0 
            \end{array}\right] \right. \right) 
    \end{align}
    must be finite and therefore so must the objective function.

    \subsection{Time-Scaled Graph Diffusion Distance}
    \label{time-scaled section}

    For any graphs $G_1$ and $G_2$, we define the Time-Scaled Graph Diffusion Distance (TSGDD) as a scaled version of the linear distance given above, with $\alpha$ fixed. Namely, let 
    \begin{align}
        \tilde{D}^2_r(G_1, G_2) &= {\left( n_1 n_2 \right)}^{-2r} \tilde{D}^2\left(G_1, G_2 \left| \alpha = {\left(\frac{n_1}{n_2}\right)}^r \right. \right) \label{eqn:TSGDD_defn}\\
        &= \inf_{P | \mathcal{C}(P)} {\left( n_1 n_2 \right)}^{-2r} {\left| \left| {\left(\frac{n_1}{n_2}\right)}^{-r} P L_1 - {\left(\frac{n_1}{n_2}\right)}^r L_2 P \right| \right|}^2_F \nonumber
    \end{align}
    The intuition for this version of the distance measure is that we are constraining the time dilation, $\alpha$, between $G_1$ and $G_2$ to be a power of the ratio of the two graph sizes. The factor ${\left(n_1 n_2 \right)}^{-2r}$ is needed to ensure this version of the distance satisfies the triangle inequality, as seen in Theorem \ref{thm:tsgdd_triangle}. 

\vspace{-5pt}
\subsection{Sparse-Diffusion Distance}
\label{sparsity_section}

We introduce the notation $\mathcal{C}(P)$ for a constraint
predicate that
must be satisfied by prolongation matrix $P$, 
which is transitive in the sense that:
\begin{align}
    \label{eqn:tran_constraint_defn}
    \mathcal{C}(P_{32}) \wedge \mathcal{C}(P_{21}) \implies \mathcal{C}(P_{32} P_{21}).
\end{align}
The simplest example is 
$\mathcal{C}(P) = \mathcal{C}_{\text{orthog}}(P) \equiv (P^T P = I)$.
However, sparsity can be introduced in transitive form by 
$\mathcal{C}(P) = \mathcal{C}_{\text{orthog}}(P) \wedge C_{\text{sparsity}}(P)$
where 
\[
\mathcal{C}_{\text{sparsity}}(P) \equiv
\Big(
\max_{i,j} \text{degree}_{i,j}(P) \leq (n_{P \text{coarse}}/n_{P \text{fine}})^s
\Big) 
\]
for some real number $s\geq 0$.
This predicate is transitive since 
\[
\max_{i,j} \text{degree}_{i,j}(P_{3 2} P_{2 1}) \leq 
   \max_{i,j} \text{degree}_{i,j}(P_{3 2} )
   \max_{i,j} \text{degree}_{i,j}(P_{2 1}) ,
\]
and since $n_2$ cancels out from the numerator and denominator
of the product of the fanout bounds. Here, $\text{degree}_{i,j}(M)$ is the total number of nonzero entries in row $i$ or column $j$ of $M$.

This transitive sparsity constraint depends on a power-law parameter $s \geq 0$. When $s=0$, there is no sparsity constraint.

Another form of sparsity constraints are those which specify a pattern on matrix entries which are allowed to be nonzero. Two simple examples (which are also transitive) are matrices which are constrained to be upper triangular, as well as matrices which are constrained to be of the form $A \otimes B$ where $A$ and $B$ are themselves both constrained to be sparse. More complicated are $n_1 \times n_2$ matrices which are constrained to be banded for some specified pattern of bands: more specifically, that there is a reordering of the rows and columns that the number of diagonal bands (of width 1, slope $\frac{n_1}{n_2}$) with nonzero entries is less than ${\left(\frac{n_1}{n_2}\right)}^q$ for some $0 \leq q < 1$. For example, linear interpolation matrices between d-dimensional grids, with non-overlapping source regions, follow this constraint.

As a final note on sparsity, we observe that any of the optimizations detailed in this work could also be performed including a sparsity term (for example, the $|\cdot|_1$-norm of the matrix $P$, calculated as $\sum_i \sum_j |p_{i j}|$ is one possibility, as are terms which penalize $t$ or $\alpha$ far from 1), rather than explicit sparsity constraints. A potential method of performing this optimization would be to start by optimizing the non-sparse version of the objective function (as detailed in Section \ref{subsec:alg_details}) and then slowly increasing the strength of the regularization term. 

\subsection{Directedness of Distance and Constraints}
    We note that this distance measure, as defined so far, is \emph{directed}: the operands $G_1$ and $G_2$ serve differing roles in the objective function. This additionally makes the constraint predicate $\mathcal{C}(P)$ ambiguous: when we state that $\mathcal{C}$ represents orthogonality, it is not clear whether we are referring to $P ^T P = I$ or $P P ^T = I$ (only one of which can be true for a non-square matrix $P$). To remove this ambiguity, we will, for the computations in the rest of this manuscript, define the distance metric to be symmetric: the distance between $G_1$ and $G_2$ with $ \left| G_1  \right| \leq \left| G_2 \right|$ is always $D(G_1, G_2)$. $P$ is then always at least as tall as it is wide, so of the two choices of orthogonality constraint we select  $P ^T P = I$.

        \subsection{Variants of Distance Measure}
        \label{subsec:dist_variants}
    Thus far we have avoided referring to this graph dissimilarity function as a ``distance metric''. As we shall see later, full optimization of Equations \ref{defn:exp_defn} and \ref{defn:linear_defn} over $\alpha$ and $P$ is too loose, in the sense that the distances $D(G_1, G_2)$, $D(G_2, G_3)$, and $D(G_1, G_3)$ do not necessarily satisfy the triangle inequality. The same is true for $\tilde{D}$. See Section \ref{subsub:tri_ineq_viol} for numerical experiments suggesting a particular parameter regime where the triangle inequality is satisfied. We thus define several restricted/augmented versions of both $D$ and $\tilde{D}$ which are guaranteed to satisfy the triangle inequality. These different versions are summarized in Table \ref{tab:dist_versions}. These variously satisfy some of the conditions necessary for generalized versions of distance metrics, including:
    \begin{itemize}
        \item Premetric: a function $d(x,y)$ for which $d(x,y) \geq 0$ and $d(x,y) = d(y,x)$ for all $x,y$.
        \item Pseudometric: As a premetric, but additionally $d(x,z) \leq d(x,y) + d(y,z)$ for all $x, y, z$.
        \item $\rho$-inframetric: As a premetric, but additionally $d(x,z) \leq \rho \left(d(x,y) + d(y,z)\right)$ and $d(x,y) = 0$ if and only if $x=y$, for all $x, y, z$.
    \end{itemize}
    Additionally, we note here that a distance measure on graphs using Laplacian spectra can at best be a pseudometric, since isospectral, non-isomorphic graphs are well-known to exist \cite{godsil1982constructing}\cite{van2003graphs}. Characterizing the conditions under which two graphs are isospectral but not isomorphic is an open problem in spectral graph theory. However, previous computational work has led to the conjecture that ``almost all'' graphs are uniquely defined by their spectra \cite{brouwer2011spectra}\cite{brouwer2009cospectral}\cite{van2009developments}, in the sense that the probability of two graphs of size $n$ being isospectral but not isomorphic goes to 0 as $n \rightarrow \infty$. Furthermore, our numerical experiments have indicated that the violation of the triangle inequality is bounded, in the sense that $D(G_1, G_3) \leq \rho * \left(D(G_1, G_2) + D(G_2, G_3) \right)$ for $\rho \approx 2.1$. This means that even in the least restricted case our similarity measure may be a $2.1$-infra-pseudometric on graphs (and, since such isospectral, non-isomorphic graphs are relatively rare, it behaves like a $2.1-$inframetric). As we will see in Section \ref{sec:theory_properties}, in some more restricted cases we can prove triangle inequalities, making our measure a pseudometric. In Section \ref{subsec:alg_details}, we discuss an algorithm for computing the optima in Equations (\ref{defn:exp_defn}) and (\ref{defn:linear_defn}). First, we discuss some theoretical properties of this dissimilarity measure.  
    \begin{table}
        \centering
        \begin{tabularx}{.9\textwidth}{| p{.11\textwidth} | p{.11\textwidth} | p{.1\textwidth} | p{.15\textwidth} | p{.25\textwidth} |}
            \hline
            $t$ & $\alpha$ & $s$ & Classification & Treatment in this manuscript \\
            \hhline{|=|=|=|=|=|}
            Fixed at $t_c < \epsilon$ & Fixed at $\alpha_c = 1$ & $s=0$ & Pseudometric & Defined in Equation \ref{defn:alpha_1_linear}. Optimized by one pass of LAP solver. Triangle inequality proven in Theorem \ref{thm:constant_alpha_thm}.\\
            \hline
            Fixed at $t_c < \epsilon$ & Fixed at $\alpha_c = (\frac{n_1}{n_2})^r$ & $s=0$ & Pseudometric & Defined in Equation \eqref{eqn:TSGDD_defn}. Optimized by one pass of LAP solver. Triangle inequality proven in Theorem \ref{thm:tsgdd_triangle}.\\
            \hline
            Fixed at $t_c < \epsilon$ & Optimized & $s=0$ & Premetric & Defined in Equation \ref{defn:linear_defn}. Optimized by Algorithm \ref{alg:linear_alg}. Triangle inequality violations examined experimentally in Section \ref{subsub:tri_ineq_viol}. \\
            \hline
            Optimized & Fixed at $\alpha_c = 1$ & $s=0$ & & When $|G_1| = |G_2|$, this is Hammond et. al's version of graph distance. \\
            \hline
            Optimized & Fixed at $\alpha_c = (\frac{n_1}{n_2})^r$ & $s=0$ & &\\
            \hline
            Optimized & Optimized & $s=0$ & Premetric & Defined in Equation \ref{defn:exp_defn}. Optimized by Algorithm \ref{alg:exponential_alg}. Graph Product upper bound proven in Theorem \ref{thm:sq_grid_lim_lem}. Triangle inequality violations examined experimentally in Section \ref{subsub:tri_ineq_viol}. Used to calculate graph distances in Sections \ref{subsub:lineage_dist} and \ref{subsub:graph_limits}.  \\
            \hline
            Fixed at $t_c < \epsilon$ & Fixed at $\alpha_c = 1$ & $s > 0$ & Pseudometric & Triangle inequality proven in Theorem \ref{thm:constant_alpha_thm}.\\
            \hline
            Fixed at $t_c < \epsilon$ & Fixed at $\alpha_c = (\frac{n_1}{n_2})^r$ & $s > 0$ & Pseudometric & Triangle inequality proven in Theorem \ref{thm:tsgdd_triangle}. \\
            \hline
            Fixed at $t_c < \epsilon$ & Optimized & $s > 0$ & & \\
            \hline
            Optimized & Fixed at $\alpha_c = 1$ & $s > 0$ & & \\
            \hline
            Optimized & Fixed at $\alpha_c = (\frac{n_1}{n_2})^r$ & $s > 0$ & &  \\
            \hline
            Optimized & Optimized & $s> 0$ & & Discussed in Section \ref{sparsity_section}. \\
            \hline
        \end{tabularx}
        \caption{Summary of this paper's investigation of different forms of our graph dissimilarity measure. In this work, we systematically explore properties of this measure given sparsity parameter $s=0$, and various regimes of $t$ (fixed at some early time, or maximized over all $t$) and $\alpha$ (fixed at $\alpha = 1$, fixed at a constant power $r$ of the ratio of graph sizes, or minimized over all $\alpha$. We leave exploration of nonzero values of the sparsity parameter to future work. }
        \label{tab:dist_versions}
    \end{table}

\vspace{-15pt}
\section{Theoretical Properties of D(G1, G2)}
In this section we prove several properties of various instances of our graph dissimilarity score, including triangle inequalities for some specific versions and an upper bound on the distance between two graph products. 
\label{sec:theory_properties}
\vspace{-35pt}
\subsection{Properties when $t$ is constant, $\alpha = 1$}
\subsubsection{Triangle Inequality}
\begin{lemma}
\label{lem:p_lem_1} For any matrices $M$ and $P$, with $P$ satisfying $P^T P = I$,\\ ${\left| \left| P M\right| \right|}^2_F \leq {\left| \left| M \right| \right|}^2_F$ and ${\left| \left| M P \right| \right|}^2_F \leq {\left| \left| M \right| \right|}^2_F$ .
\end{lemma}
\begin{proof}
Suppose without loss of generality that $P^T P = I$. Then:
\begin{enumerate}
    \item ${\left| \left| P M\right| \right|}^2_F = \Tr [ M^T P^T P M ] = \Tr[M^T M] = {\left| \left|M\right| \right|}^2_F$ 
    \item If $P^T P = I$, then letting $P P^T = \Pi$, $\Pi$ is a projection operator satisfying $\Pi^T = \Pi = \Pi^2$. Then, 
\end{enumerate}
\vspace{-5pt}
    \begin{align*}
    {\left| \left| M P \right| \right|}^2_F &= \Tr [ P^T M^T M P  ] \\
    &= \Tr[M^T M P P^T] = \Tr[M^T M \Pi] 
    \end{align*}
    
    \begin{equation}
    \begin{split}
    {\left| \left| M \right| \right|}^2_F = \Tr [M^T M ] &= \Tr[M^T M (\Pi + (I - \Pi))] \\
    &= \Tr[M^T M \Pi] + \Tr[M^T M (I - \Pi)] \\
    &= \Tr[M^T M \Pi]+ \Tr[M^T M {(I - \Pi)}^2] \\
    &=  {\left| \left| M P \right| \right|}_F^2 + {\left| \left| M (I - \Pi) \right| \right|}_F^2 \\
    &\geq {\left| \left| M P \right| \right|}_F^2
    \end{split}
    \end{equation}
\end{proof}
\begin{theorem}
\label{thm:constant_alpha_thm}
$\tilde{D}^2$ satisfies the triangle inequality for $\alpha = 1$. 
\end{theorem}
\begin{proof}
Let $G_1, G_2, G_3$ be simple graphs, with Laplacians $L_1, L_2, L_3$. Let 
\begin{align}
P_{31} = \arg\inf_{P | \mathcal{C}(P)} {\left| \left| P L_1 - L_3 P \right| \right|}_F^2 
\end{align}
Then
\begin{equation}
\begin{split}
    \tilde{D}^2(G_1, G_3 \left| \alpha= 1\right.) &= {\left| \left| P_{31} L_1 - L_3 P_{31} \right| \right|}_F^2 = \inf_{P | \mathcal{C}(P)} {\left| \left| P L_1 - L_3 P \right| \right|}_F^2 \label{defn:alpha_1_linear} \\
    &\leq \inf_{P_{32}, P_{21} | \mathcal{C}(P_{32} P_{21})} {\left| \left| P_{32} P_{21} L_1 - L_3 P_{32} P_{21} \right| \right|}_F^2 \\
    \end{split}
    \end{equation}
    where we write $\mathcal{C}(P_{32} P_{21})$ to signify that the product $P_{32} P_{21}$ 
    satisfies the original transitive constraints on $P$,
    e.g. orthogonality and/or sparsity. 
    Since the constraint predicate $\mathcal{C}(P)$ satisfies Equation \eqref{eqn:tran_constraint_defn}, then so does their product, so we may write
    \begin{equation}
    \begin{split}
    \tilde{D}(G_1, G_3 \left| \alpha= 1\right.) &\leq \inf_{P_{32} | \mathcal{C}(P_{32})} \inf_{P_{21} | \mathcal{C}(P_{21})} {\left| \left| P_{32} P_{21} L_1 - L_3 P_{32} P_{21} \right| \right|}_F \\
    &= \inf_{P_{32} | \mathcal{C}(P_{32})} \inf_{P_{21} | \mathcal{C}(P_{21})} {\left| \left| P_{32} P_{21} L_1 - P_{32} L_2 P_{21} \right. \right.} \\
    & {\left. \left. \quad \quad \qquad + \quad P_{32} L_2 P_{21} - L_3 P_{32} P_{21} \right| \right|}_F \\
    &\leq \inf_{P_{32} | \mathcal{C}(P_{32})} \inf_{P_{21} | \mathcal{C}(P_{21})} \left( {\left| \left| P_{32} P_{21} L_1 - P_{32} L_2 P_{21} \right| \right|}_F \right.  \\
    & \quad \quad \qquad + \quad \left. {\left| \left| P_{32} L_2 P_{21} - L_3 P_{32} P_{21} \right| \right|}_F \right) \\
    &=\inf_{P_{32} | \mathcal{C}(P_{32})} \inf_{P_{21} | \mathcal{C}(P_{21})} \left( {\left| \left| P_{32} \left( P_{21} L_1 - L_2 P_{21} \right) \right| \right|}_F \right. \\
    &\quad \quad \qquad + \quad \left. {\left| \left| \left( P_{32} L_2 - L_3 P_{32} \right) P_{21} \right| \right|}_F \right)\\
    \end{split}
    \end{equation}
    By Lemma \ref{lem:p_lem_1}, 
    \begin{equation}
    \begin{split}
    \tilde{D}(G_1, G_3 \left| \alpha= 1\right.) &\leq \inf_{P_{32} | \mathcal{C}(P_{32})}  \inf_{P_{21} | \mathcal{C}(P_{21})} \left( {\left| \left|  P_{21} L_1 - L_2 P_{21} \right| \right|}_F \right.  \\
    & \quad \quad \qquad + \quad \left. {\left| \left| P_{32} L_2 - L_3 P_{32} \right| \right|}_F \right) \\
    &=  \inf_{P_{21} | \mathcal{C}(P_{21})} {\left| \left|  P_{21} L_1 - L_2 P_{21} \right| \right|}_F  \\
    & \quad \quad \qquad + \quad \inf_{P_{32} | \mathcal{C}(P_{32})} {\left| \left| P_{32} L_2 - L_3 P_{32} \right| \right|}_F \\
    &=  \tilde{D}(G_1, G_2 \left| \alpha= 1\right.) + \tilde{D}(G_2, G_3 \left| \alpha= 1\right.) 
    \end{split}
    \end{equation}
\end{proof}
We note that in this proof we use $L_1, L_2,$ and $L_3$ (making this the small-$t$ or linear version of the objective function), but the same argument holds when all three are replaced with $e^{t L_i}$, so we also have  
\begin{corollary}
\label{cor:constant_alpha_exp}
$D$ satisfies the triangle inequality for $\alpha = 1$. 
\end{corollary}
\begin{proof}
By the same calculation as in Theorem \ref{thm:constant_alpha_thm}, with all $L_i$ replaced by $e^{t_c L_i}$,  we have
\begin{align}
     D\left( \left. G_1, G_3 \right| t_c, \alpha = 1 \right) &\leq  D(G_1, G_2\left| t_c, \alpha= 1 \right) + D(G_2, G_3\left| t_c, \alpha= 1 \right)\\
    \nonumber
\end{align}
for any constant $t_c$. Then, letting 
\begin{align}
    t_{13} = \arg \sup_{t_c} D\left( \left. G_1, G_3 \right| t_c, \alpha = 1 \right)
\end{align}
we have: 
\begin{equation}
\begin{split}
     D\left( \left. G_1, G_3 \right| \alpha = 1 \right) &= \sup_{t_c} D\left( \left. G_1, G_3 \right| t_c, \alpha = 1 \right) \\ 
     &= D\left( \left. G_1, G_3 \right| t_{13}, \alpha = 1 \right) \\
     &\leq D(G_1, G_2\left| t_{13}, \alpha = 1 \right) + D(G_2, G_3\left| t_{13}, \alpha = 1 \right) \\
     &\leq \sup_{t_c} D\left( \left. G_1, G_2 \right| t_c, \alpha = 1 \right) \\
    & \quad \quad \qquad + \quad \sup_{t_c} D\left( \left. G_2, G_3 \right| t_c, \alpha = 1 \right) \\
     &= D\left( \left. G_1, G_2 \right| \alpha = 1 \right) + D\left( \left. G_2, G_3 \right| \alpha = 1 \right)
\end{split}
\end{equation}
\end{proof}

\subsection{Properties when $t$ is constant, $\alpha = {\left(\frac{n_1}{n_2}\right)}^r$}
Recall our definition of the \emph{Time-Scaled Graph Diffusion Distance} from Equation \eqref{eqn:TSGDD_defn}, for $\alpha = {\left( \frac{n_1}{n_2} \right)}^r$, $r \in \mathbb{R}$: 
\begin{equation}
\begin{split}
    \tilde{D}^2_r(G_1, G_2) &= {\left( n_1 n_2\right)}^{-2r} \tilde{D}^2 \left( G_1, G_2 \left| \alpha = {\left( \frac{n_1}{n_2} \right)}^r \right. \right) \\
    &= \inf_{P|\mathcal{C}(P)} {\left( n_1 n_2\right)}^{-2r} {\left|\left| {\left( \frac{n_1}{n_2} \right)}^{-r} P L_1 - {\left( \frac{n_1}{n_2} \right)}^{r} L_2 P  \right|\right|}_F^2
\end{split}
\end{equation}

\begin{theorem}
\label{thm:tsgdd_triangle}
The TSGDD, as defined above, satisfies the triangle inequality.
\end{theorem}
\begin{proof}

 As above, let $G_1, G_2, G_3$ be three graphs with $n_i = \left| G_i \right|$ and $n_1 \leq n_2 \leq n_3$, and let $L_i$ be the Laplacian of $G_i$. Let $\mathcal{C}(P)$ represent a transitive constraint predicate, also as described previously. Then, for a constant $r \in \mathbb{R}$, we have:
\begin{align*}
    &\tilde{D}_r(G_1, G_3) = \\
    & \qquad \inf_{P|\mathcal{C}(P)} {\left( n_1 n_3\right)}^{-r} {\left|\left| {\left( \frac{n_1}{n_3} \right)}^{-r} P L_1 - {\left( \frac{n_1}{n_3} \right)}^{r} L_3 P  \right|\right|}_F \nonumber \\
    & \qquad \leq \inf_{P_{32}, P_{21} | \mathcal{C}(P_{32} P_{21})} {\left( n_1 n_3\right)}^{-r} {\left|\left| {\left( \frac{n_1}{n_3} \right)}^{-r} P_{32} P_{21} L_1 - {\left( \frac{n_1}{n_3} \right)}^{r} L_3 P_{32} P_{21}  \right|\right|}_F \nonumber \\
    \intertext{under the assumption, as in Equation \eqref{eqn:tran_constraint_defn}, that $  \mathcal{C}(P_{32}) \wedge \mathcal{C}(P_{21}) \implies \mathcal{C}(P_{32} P_{21}) $, }
    \end{align*}
    \vspace{-1.0cm}
    \begin{align*}
    &\tilde{D}_r(G_1, G_3) \leq \\
    &\inf_{P_{32}, P_{21} | \mathcal{C}(P_{32}) \wedge \mathcal{C}(P_{21})} {\left( n_1 n_3\right)}^{-r} {\left|\left| {\left( \frac{n_1}{n_3} \right)}^{-r} P_{32} P_{21} L_1 - {\left( \frac{n_1}{n_3} \right)}^{r} L_3 P_{32} P_{21}  \right|\right|}_F \nonumber \\
    &= \inf_{P_{32}, P_{21} | \mathcal{C}(P_{32}) \wedge \mathcal{C}(P_{21})} {\left( n_1 n_3\right)}^{-r} {\left|\left| {\left( \frac{n_1}{n_3} \right)}^{-r} P_{32} P_{21} L_1 - {\left( \frac{n_1 n_3}{n_2^2} \right)}^{r} P_{32} L_2 P_{21}  \right. \right. } \nonumber \\*
    &\null \qquad + \qquad {\left. \left. {\left( \frac{n_1 n_3}{n_2^2} \right)}^{r} P_{32} L_2 P_{21} - {\left( \frac{n_1}{n_3} \right)}^{r} L_3 P_{32} P_{21}  \right|\right|}_F \\
    &\leq  \inf_{P_{32}, P_{21} | \mathcal{C}(P_{32}) \wedge \mathcal{C}(P_{21})} {\left( n_1 n_3\right)}^{-r} {\left|\left| {\left( \frac{n_1}{n_3} \right)}^{-r} P_{32} P_{21} L_1 - {\left( \frac{n_1 n_3}{n_2^2} \right)}^{r} P_{32} L_2 P_{21}  \right|\right|}_F \\
    &\null \qquad + \qquad {\left( n_1 n_3\right)}^{-r} {\left|\left| {\left( \frac{n_1 n_3}{n_2^2} \right)}^{r} P_{32} L_2 P_{21} - {\left( \frac{n_1}{n_3} \right)}^{r} L_3 P_{32} P_{21}   \right|\right|}_F \\
    &= \inf_{P_{32}, P_{21} | \mathcal{C}(P_{32}) \wedge \mathcal{C}(P_{21})} {\left( n_1 n_3\right)}^{-r} {\left( \frac{n_3}{n_2} \right)}^{r} {\left|\left| {\left( \frac{n_1}{n_2} \right)}^{-r} P_{32} P_{21} L_1 - {\left( \frac{n_1}{n_2} \right)}^{r} P_{32} L_2 P_{21}  \right|\right|}_F \\
    &\null \qquad + \qquad {\left( n_1 n_3\right)}^{-r} {\left( \frac{n_1}{n_2} \right)}^{r} {\left|\left| {\left( \frac{n_2}{n_3} \right)}^{-r} P_{32} L_2 P_{21} - {\left( \frac{n_2}{n_3} \right)}^{r} L_3 P_{32} P_{21}   \right|\right|}_F \\
    &= \inf_{P_{32}, P_{21} | \mathcal{C}(P_{32}) \wedge \mathcal{C}(P_{21})} {\left( n_1 n_2\right)}^{-r} {\left|\left| {\left( \frac{n_1}{n_2} \right)}^{-r} P_{32} P_{21} L_1 - {\left( \frac{n_1}{n_2} \right)}^{r} P_{32} L_2 P_{21}  \right|\right|}_F \\
    &\null \qquad + \qquad {\left( n_2 n_3\right)}^{-r} {\left|\left| {\left( \frac{n_2}{n_3} \right)}^{-r} P_{32} L_2 P_{21} - {\left( \frac{n_2}{n_3} \right)}^{r} L_3 P_{32} P_{21}   \right|\right|}_F
    \intertext{By Lemma \ref{lem:p_lem_1},}
    \end{align*}
    \begin{align*}
    \tilde{D}_r(G_1, G_3) &\leq  \inf_{P_{32}, P_{21} | \mathcal{C}(P_{32}) \wedge \mathcal{C}(P_{21})} {\left( n_1 n_2\right)}^{-r} {\left|\left| {\left( \frac{n_1}{n_2} \right)}^{-r} P_{21} L_1 - {\left( \frac{n_1}{n_2} \right)}^{r} L_2 P_{21}  \right|\right|}_F \\*
    &\null \qquad + \qquad {\left( n_2 n_3\right)}^{-r} {\left|\left| {\left( \frac{n_2}{n_3} \right)}^{-r} P_{32} L_2 - {\left( \frac{n_2}{n_3} \right)}^{r} L_3 P_{32}   \right|\right|}_F \\
    &= \inf_{P_{21} | \mathcal{C}(P_{21})} {\left( n_1 n_2\right)}^{-r} {\left|\left| {\left( \frac{n_1}{n_2} \right)}^{-r} P_{21} L_1 - {\left( \frac{n_1}{n_2} \right)}^{r} L_2 P_{21}  \right|\right|}_F \\*
    &\null \qquad + \qquad \inf_{P_{32} | \mathcal{C}(P_{32})} {\left( n_2 n_3\right)}^{-r} {\left|\left| {\left( \frac{n_2}{n_3} \right)}^{-r} P_{32} L_2 - {\left( \frac{n_2}{n_3} \right)}^{r} L_3 P_{32}   \right|\right|}_F \\
    &= \tilde{D}_r(G_1, G_2) + \tilde{D}_r(G_2, G_3)
    \nonumber
    \intertext{and so}
    &\tilde{D}_r(G_1, G_3) \leq \tilde{D}_r(G_1, G_2) + \tilde{D}_r(G_2, G_3)
\end{align*}
for any fixed $r \in \mathbb{R}$.
\end{proof}

Note that in the proofs of Theorem \ref{thm:constant_alpha_thm}, Theorem \ref{thm:tsgdd_triangle}, and Corollary \ref{cor:constant_alpha_exp}, we assume that the constraint predicate $\mathcal{C}(P)$ includes at least orthogonality (so that we may apply Lemma \ref{lem:p_lem_1}). However, this constraint predicate could be more strict, e.g. include both orthogonality and sparsity. Hence these statements also apply to the $s > 0$ cases in Table \ref{tab:dist_versions}, which we do not otherwise consider in this work: in our numerical experiments we (for reasons of computational simplicity) only require our optimization over $P$ be orthogonally constrained.

\subsection{Upper Bounds for Graph Products}
\label{graph_product_bound}

We now consider the case where we want to compute the distance of two graph box products, i.e. $D \left(\mathbf{G}_1, \mathbf{G}_2 \right)$ where
\begin{align}
    \mathbf{G}_1 = G_1^{(1)} \Box G_1^{(2)} \quad \text{and} \quad 
    \mathbf{G}_2 = G_2^{(1)} \Box G_2^{(2)}
\end{align}
and 
\begin{equation}
\begin{split}
    P^{(1)} = \arg \inf_{P_c \left| \mathcal{C} (P_c) \right. } D \left( G_1^{(1)}, G_2^{(1)} | t_c, \alpha_c, P_c \right) \\
    P^{(2)}= \arg \inf_{P_c \left| \mathcal{C} (P_c) \right. } D \left( G_1^{(2)}, G_2^{(2)} | t_c, \alpha_c, P_c \right)
\end{split}
\end{equation}

are known for some $t_c, \alpha_c$. Previous work \cite{scott2018multilevel} included a proof of a similar inequality for the small-$t$ (``linear'') case of our objective function. 
\begin{theorem}
\label{thm:sq_grid_lim_lem}
Let $\mathbf{G}_1$ and $\mathbf{G}_2$ be graph box products as described above, and for a graph $G$ let $L(G)$ be its Laplacian. For fixed $t=t_c$, $\alpha = \alpha_c$, $P^{(i)}$ as given above, for any $\lambda \in [0,1]$, we have 
\begin{equation}
\begin{aligned}
    &\inf_{P_c \left| \mathcal{C}(P_c) \right. } D \left(\mathbf{G}_1, \mathbf{G}_2  \right) \quad \leq \quad \\
    &\qquad \lambda \left({\left| \left| e^{\frac{t_c}{\alpha_c} L(G^{(2)}_1)}  \right| \right|}_F + 
    {\left| \left| e^{t_c \alpha_c L(G^{(2)}_2)}  \right| \right|}_F \right) D \left( G_1^{(1)}, G_2^{(1)} | P^{(1)} \right) \\
    &\qquad \qquad + \left( 1 - \lambda \right) \left({\left| \left| e^{\frac{t_c}{\alpha_c} L(G^{(1)}_1)}  \right| \right|}_F + 
    {\left| \left| e^{t_c \alpha_c L(G^{(1)}_2)}  \right| \right|}_F \right) D \left( G_1^{(2)}, G_2^{(2)} | P^{(2)} \right)
\end{aligned}
\end{equation}
where all distances are evaluated at $t=t_c$, $\alpha = \alpha_c$, but we have omitted those terms for simplicity of notation.
\end{theorem}
\begin{proof}
For graph products $\mathbf{G}_i$, we have 
\begin{equation}
\begin{aligned}
    L(\mathbf{G}_i) &= L(G^{(1)}_i) \oplus L(G^{(2)}_i) \\
    &= \left( L(G^{(1)}_i) \otimes I_{\left| L(G^{(2)}_i) \right|} \right) + \left( 
    I_{\left| L(G^{(1)}_i) \right|} \otimes L(G^{(2)}_i)
    \right) 
\end{aligned}
\end{equation}
(this fact can be easily verified from the formula for the adjacency matrix of a graph box product, given in the definition in Section \ref{subsec:defns}), and so
\begin{equation}
\begin{aligned}
    \exp \left[{c L(\mathbf{G}_i)}\right] &= \exp \left[c {\left( L(G^{(1)}_i) \otimes I_{\left| L(G^{(2)}_i) \right|}  \right) + \left(
    I_{\left| L(G^{(1)}_i) \right|} \otimes L(G^{(2)}_i) \right)}\right].\\
\end{aligned}
\end{equation}
Because $A \otimes I_{|B|}$ and $I_{|A|} \otimes B$ commute for any $A$ and $B$,
\begin{equation}
\begin{aligned}
    \exp \left[{c L(\mathbf{G}_i)}\right] &= \exp \left[ c {\left( L(G^{(1)}_i) \otimes I_{\left| L(G^{(2)}_i) \right|} \right)}\right] \exp \left[ c {\left( 
    I_{\left| L(G^{(1)}_i) \right|} \otimes L(G^{(2)}_i) \right)}\right] \\
    &= \left( \exp \left[ c { L(G^{(1)}_i) } \right] \otimes I_{\left| L(G^{(2)}_i) \right|}  \right) 
    \left(I_{\left| L(G^{(1)}_i) \right|}  \otimes \exp \left[ c { L(G^{(2)}_i) }\right]\right) \\
    &= \exp \left[ c { L(G^{(1)}_i) } \right] \otimes  \exp \left[ c { L(G^{(2)}_i) }\right]\
\end{aligned}
\end{equation}
We will make the following abbreviations:
\[
\begin{array}{ccc}
    \mathbf{E}_1 = e^{\frac{t_c}{\alpha_c} L(\mathbf{G}_1)} & E^{(1)}_1 = e^{\frac{t_c}{\alpha_c} L(G^{(1)}_1)} &  E^{(2)}_1 = e^{\frac{t_c}{\alpha_c} L(G^{(2)}_1)} \\
     \mathbf{E}_2 = e^{t_c \alpha_c L(\mathbf{G}_2)} & E^{(1)}_2 = e^{t_c \alpha_c L(G^{(1)}_2)} & E^{(2)}_2 = e^{t_c \alpha_c L(G^{(2)}_2)}
\end{array}
\]
Then,
\begin{align}
\inf_{P | \mathcal{C}(P)} D \left(\mathbf{G}_1, \mathbf{G}_2 \right) &\leq D \left(\mathbf{G}_1, \mathbf{G}_2 | P^{(1)} \otimes P^{(2)} \right) \\
&=  {\left| \left| \left( P^{(1)} \otimes P^{(2)} \right) \mathbf{E}_1 - \mathbf{E}_2 \left( P^{(1)} \otimes P^{(2)} \right)  \right| \right|}_F \nonumber
\end{align}
\newpage
\begin{align}
    &= {\left| \left| \left( P^{(1)} \otimes P^{(2)} \right) \left(E^{(1)}_1 \otimes E^{(2)}_1 \right) - \left(E^{(1)}_2 \otimes E^{(2)}_2 \right) \left( P^{(1)} \otimes P^{(2)} \right)  \right| \right|}_F \nonumber \\
    &= {\left| \left| \left( P^{(1)} E^{(1)}_1  \otimes P^{(2)} E^{(2)}_1 \right) - \left(E^{(1)}_2 P^{(1)} \otimes E^{(2)}_2 P^{(2)} \right)  \right| \right|}_F^2 \nonumber \\
    &= {\left| \left| \left( P^{(1)} E^{(1)}_1  \otimes P^{(2)} E^{(2)}_1 \right) - \left( P^{(1)} E^{(1)}_1  \otimes E^{(2)}_2 P^{(2)} \right) \right. \right.} \label{eqn:cross_term_eqn} \\*
    & \quad + {\left. \left. \left( P^{(1)} E^{(1)}_1  \otimes E^{(2)}_2 P^{(2)} \right) - \left(E^{(1)}_2 P^{(1)} \otimes E^{(2)}_2 P^{(2)} \right)  \right| \right|}_F \nonumber \\
    & \leq {\left| \left| \left( P^{(1)} E^{(1)}_1  \otimes P^{(2)} E^{(2)}_1 \right) - \left( P^{(1)} E^{(1)}_1  \otimes E^{(2)}_2 P^{(2)} \right) \right| \right|}_F \nonumber \\
    & \quad + {\left| \left| \left( P^{(1)} E^{(1)}_1  \otimes E^{(2)}_2 P^{(2)} \right) - \left(E^{(1)}_2 P^{(1)} \otimes E^{(2)}_2 P^{(2)} \right)  \right| \right|}_F \nonumber \\
    &= {\left| \left|  P^{(1)} E^{(1)}_1  \otimes  \left( P^{(2)} E^{(2)}_1 -  E^{(2)}_2 P^{(2)} \right) \right| \right|}_F \\ & \quad \quad + \quad  {\left| \left| \left( P^{(1)} E^{(1)}_1  - E^{(1)}_2 P^{(1)} \right) \otimes E^{(2)}_2 P^{(2)} \right| \right|}_F \nonumber \\
    &= {\left| \left|  P^{(1)} E^{(1)}_1 \right| \right|}_F {\left| \left|   P^{(2)} E^{(2)}_1 -  E^{(2)}_2 P^{(2)} \right| \right|}_F \\ & \quad \quad + \quad {\left| \left| P^{(1)} E^{(1)}_1  - E^{(1)}_2 P^{(1)}  \right| \right|}_F  {\left| \left|  E^{(2)}_2 P^{(2)} \right| \right|}_F . \nonumber  
    \end{align}
    By Lemma \ref{lem:p_lem_1},
    \begin{equation}
    \begin{split}
    \inf_{P | \mathcal{C}(P)} D \left(\mathbf{G}_1, \mathbf{G}_2 \right) &\leq {\left| \left| E^{(1)}_1 \right| \right|}_F {\left| \left|   P^{(2)} E^{(2)}_1 -  E^{(2)}_2 P^{(2)} \right| \right|}_F \\ & \quad \quad + \quad  {\left| \left| P^{(1)} E^{(1)}_1  - E^{(1)}_2 P^{(1)}  \right| \right|}_F  {\left| \left|  E^{(2)}_2 \right| \right|}_F . 
    \end{split}
    \end{equation}
    \begin{align}
    \intertext{If we instead use
    $\left( E^{(1)}_2 P^{(1)} \otimes P^{(2)} E^{(2)}_1 \right)$ as the cross term in Equation
    \eqref{eqn:cross_term_eqn}, we have }
    \inf_P D \left(\mathbf{G}_1, \mathbf{G}_2 \right) \leq {\left| \left| E^{(1)}_2 \right| \right|}_F {\left| \left|   P^{(2)} E^{(2)}_1 -  E^{(2)}_2 P^{(2)} \right| \right|}_F  \\ \quad \quad + \quad  {\left| \left| P^{(1)} E^{(1)}_1  - E^{(1)}_2 P^{(1)}  \right| \right|}_F  {\left| \left|  E^{(2)}_1 \right| \right|}_F \nonumber
\end{align}
A linear combination of these two bounds gives us the desired bound. 
\end{proof}
\noindent This has the additional consequence that 
\begin{equation}
    \begin{aligned}
    &\inf_{P_c \left| \mathcal{C}(P_c) \right.} D \left(\mathbf{G}_1, \mathbf{G}_2  \right) \quad \leq  & \\
    &\min \left[ \left({\left| \left| e^{\frac{t_c}{\alpha_c} L(G^{(2)}_1)}  \right| \right|}_F + 
    {\left| \left| e^{t_c \alpha_c L(G^{(2)}_2)}  \right| \right|}_F \right) D \left( G_1^{(1)}, G_2^{(1)} | P^{(1)} \right), \right. \\
    & \quad \quad  \left.  \left({\left| \left| e^{\frac{t_c}{\alpha_c} L(G^{(1)}_1)}  \right| \right|}_F + 
    {\left| \left| e^{t_c \alpha_c L(G^{(1)}_2)}  \right| \right|}_F \right) D \left( G_1^{(2)}, G_2^{(2)} | P^{(2)} \right) \right]
    \end{aligned}
\end{equation}
Additionally, if 
\begin{align}
    E^{(1)}_i = E^{(2)}_i \text{ for } i \in 1,2 \quad \text{and} \quad P^{(1)} = P^{(2)},
\end{align}
This reduces further to 
\begin{align}
D \left(\mathbf{G}_1, \mathbf{G}_2 | P^{(1)} \otimes P^{(1)} \right)
 &\leq \min \left( {\left| \left|  E^{(1)}_1 \right| \right|}_F , {\left| \left|  E^{(1)}_2 \right| \right|}_F \right) {\left| \left|   P^{(1)} E^{(1)}_1 -  E^{(1)}_2 P^{(1)} \right| \right|}_F
 \end{align}
 and so
 \begin{align}
 & D \left( \left. G_1^{(1)} \Box G_1^{(1)}, G_2^{(1)} \Box G_2^{(1)} \right| t_c, \alpha_c \right) \label{eqn:graph_product_special_case} \\ 
 &\null \quad \quad \quad \leq \min \left( {\left| \left| e^{\frac{t_c}{a_c} L(G^{(1)}_1)} \right| \right|}_F , {\left| \left| e^{t_c a_c L(G^{(1)}_2)} \right| \right|}_F \right) D \left( G_1^{(1)}, G_2^{(1)} \left. \right| t_c, \alpha_c \right) \nonumber 
\end{align}
An example of such a graph sequence is the sequence of two-dimensional square grids, which are each the box product of two identical one-dimensional grids i.e. path graphs: $\text{Sq}_n = \text{Pa}_n \Box \text{Pa}_n$.
\subsection{Theory Summary}

Triangle inequalities are proven for some members of the proposed family of 
graph distortion or ``distance'' measures, 
including infinitesimal and finite diffusion time,
a power law for sparsity,
and/or a power law for time conversion factor between coarse and fine scales.
However, the case of an optimal (not power law) time conversion factor $\alpha$ needs to be investigated
by numerical experiment, and that requires new algorithms, introduced in Section \ref{sec:alg}. We also show that in the case of distances between graph box products, optimization over $P$ for the product graphs is bounded above by a monotonic function of the optimum over the component graphs.

\section{Numerical Methods for Optimal Time Conversion, $\alpha$}
\label{sec:alg}
Optimizing the $\alpha$ parameter for conversion
between coarse and fine time scales in the proposed
family of graph distance measures, in addition
to optimizing the prolongation matrix $P$ under
transitive constraints $\mathcal{C}(P)$, 
is a nontrivial numerical problem that 
in our experience seems to require new methods.
We develop such methods here and apply them
to investigate the resulting graph ``distance'' measure
in the next section.

\subsection{Algorithm Development}
\label{subsec:alg_details}

\begin{algorithm}
\caption{Algorithm for computing $\inf_{P , \alpha} \tilde{D}^2$}
\label{alg:linear_alg}
\begin{algorithmic}[1]
\Procedure {LAP-Solve-Linear}{$\lambda^{(1)}$, $\lambda^{(2)}$, $\alpha$}
    \State Compute cost matrix $C_{ij} = {\left( 
        \frac{1}{\alpha} \lambda^{(1)}_j - \alpha \lambda^{(2)}_i
    \right)}^2$
    \State Use an external solver for the LAP to find the matching matrix $M$ which minimizes total cost $\sum_{i,j} M_{ij}C_{ij}$.
    \State Return $M$.
\EndProcedure
\Procedure {Merge-Soln}{$M^{(1)}$, $M^{(2)}$,$\lambda^{(1)}$, $\lambda^{(2)}$, $\alpha_1$, $\alpha_2$}
    \State Construct $\tilde{e}^{(1)}$ and $\tilde{e}^{(2)}$ by removing $\lambda^{(1)}_j$ and $\lambda^{(2)}_i$ whenever $M^{(1)}_{ij} = M^{(2)}_{ij} = 1$ (keeping track of the original indices of each of the remaining values).
    \State Compute $\alpha^*$ such that $c\left(M^{(1)}, \lambda^{(1)}, \lambda^{(2)}, \alpha^* \right) = c\left(M^{(2)}, \lambda^{(1)}, \lambda^{(2)}, \alpha^* \right)$
    \If{$ \alpha^* \notin [\alpha_1, \alpha_2]$, or $\tilde{e}^{(1)}$ is empty}:
    \State Mark the interval $[\alpha_1, \alpha_2]$ as searched. 
    \State return $(M^{(1)}, \alpha^*)$ or $(M^{(2)}, \alpha^*)$.
    \Else 
    \State Compute $M_\text{sub}$ as the solution given by  $\textsc{LAP-Solve-Linear}(\tilde{e}^{(2)}, \tilde{e}^{(2)}, \alpha^*)$.
    \State Construct a full-size solution $M$ for the original problem by starting with $M = M^{(1)} \circ M^{(2)}$ (the elementwise product of $M^{(1)}$ and $M^{(2)}$, which is 1 exactly where they agree and 0 otherwise) and inserting a 1 for any pair $i,j$ which are matched by $M_\text{sub}$.
    \State return $(M, \alpha^*)$.
    \EndIf
\EndProcedure
\Procedure{$\tilde{D}^2$}{$G_1, G_2$}
\State Compute graph Laplacians $L_1 = L(G_1)$ and $L_2 = L(G_2)$.
\State Compute $\lambda^{(1)} = \textsc{Eigenvals}(L_1)$ and  $\lambda^{(2)} = \textsc{Eigenvals}(L_2)$.
\State Initialize the list of known optimal matchings with $(n_2 \times n_1)$ matrices $\left[ \begin{array}{c}
I_{n_1} \\ 0
\end{array} \right]$ and $\left[ \begin{array}{c}
0 \\ I_{n_1}  
\end{array} \right]$ and associated $\alpha_\text{low}$ and $\alpha_\text{high}$.

\While{there are uncovered intervals within $[\alpha_\text{low}, \alpha_\text{high}]$:}
\State Order the current known pairs $(M_i, \alpha_i)$ by $\alpha$. 
\For{each $(M_i,\alpha_i), (M_{i+1},\alpha_{i+1})$ in this list:}
    \State Compute $(M, \alpha^*) = \textsc{Merge-Soln}(M_{i},M_{i+1},\lambda^{(1)},\lambda^{(2)},\alpha_1,\alpha_2)$.
    \If{$M = M_i$ or $M = M_{i+1}$}
        \State Mark the interval $[\alpha_{i}, \alpha_{i+1}]$ as searched.
    \Else
        \State Insert $(M, \alpha^*)$ into the appropriate place in the list of $M_i$. 
    \EndIf
\EndFor
\EndWhile
\State Find the unique global minimum of the cost of each found matching as a function of $\alpha$.
\State Let $M^*$ be the matching which is minimal at its optimal $\alpha$ (out of all matchings found).
\State Return the minimum cost $\tilde{D}^2(G_1, G_2 | M^*, \alpha^*)$.
\State (alternately, return the entire frontier $\left\{ \left(M_{1}, \alpha_1\right) \ldots \left(M_{m}, \alpha_m\right) \right\}$).
\EndProcedure
\end{algorithmic}
\end{algorithm}

\begin{algorithm}
\caption{Algorithm for computing $\sup_t \inf_{P , \alpha} D^2$}
\label{alg:exponential_alg}
\begin{algorithmic}
\Procedure {LAP-Solve-Exponential}{$\lambda^{(1)}$, $\lambda^{(2)}$, $\alpha$, t}
    \State Compute cost matrix $C_{ij} = {\left( 
        e^{\frac{t}{\alpha} \lambda^{(1)}_j} - e^{\alpha t \lambda^{(2)}_i}
    \right)}^2$
    \State Use an external solver for the LAP to find the matching matrix $M$ which minimizes total cost $\sum_{i,j} M_{ij}C_{ij}$.
    \State Return $M$.
\EndProcedure
\Procedure {Merge-Soln}{$M^{(1)}$, $M^{(2)}$,$\lambda^{(1)}$, $\lambda^{(2)}$, $\alpha_1$, $\alpha_2$, t}
    \State Construct $\tilde{e}^{(1)}$ and $\tilde{e}^{(2)}$ by removing $\lambda^{(1)}_j$ and $\lambda^{(2)}_i$ whenever $M^{(1)}_{ij} = M^{(2)}_{ij} = 1$ (keeping track of the original indices of each of the remaining values).
    \State Compute $\alpha^*$ such that $c\left(M^{(1)}, \lambda^{(1)}, \lambda^{(2)}, \alpha^*, t \right) = c\left(M^{(2)}, \lambda^{(1)}, \lambda^{(2)}, \alpha^*, t \right)$
    \If{$ \alpha^* \notin [\alpha_1, \alpha_2]$, or $\tilde{e}^{(1)}$ is empty}:
    \State Mark the interval $[\alpha_1, \alpha_2]$ as searched. 
    \State return $(M^{(1)}, \alpha^*)$ or $(M^{(2)}, \alpha^*)$.
    \Else 
    \State Compute $M_\text{sub}$ as the solution given by  $\textsc{LAP-Solve-Exponential}(\tilde{e}^{(2)}, \tilde{e}^{(2)}, \alpha^*, t)$.
    \State Construct a full-size solution $M$ for the original problem by starting with $M = M^{(1)} \circ M^{(2)}$ (the elementwise product of $M^{(1)}$ and $M^{(2)}$, which is 1 exactly where they agree and 0 otherwise) and inserting a 1 for any pair $i,j$ which are matched by $M_\text{sub}$.
    \State return $(M, \alpha^*)$.
    \EndIf
\EndProcedure

\Procedure{T-Step}{$\lambda^{(1)}$, $\lambda^{(2)}$, \textsc{Frontier} $= \left\{ \left( M_1, \alpha_1 \right), \left( M_2, \alpha_2 \right) \ldots \left( M_m, \alpha_m \right) \right\}$, t}
\While{\textsc{Frontier} is still expanding}
\For{each $(M_i,\alpha_i), (M_{i+1},\alpha_{i+1})$ in \textsc{Frontier}}
    \If{$M_i$ and $M_{i+1}$ have already been checked}
        \State pass
    \Else
        \State Compute $(M, \alpha^*) = \textsc{Merge-Soln}(M_{i},M_{i+1},\lambda^{(1)},\lambda^{(2)},\alpha_{i},\alpha_{i+1}, t)$.
        \If{$M = M_i$ or $M = M_{i+1}$}
            \State pass
        \Else
            \State Insert $(M, \alpha^*)$ into the appropriate place in \textsc{Frontier}. 
        \EndIf
    \EndIf
\EndFor
\EndWhile
\State Return \textsc{Frontier}.
\EndProcedure

\Procedure{$D^2$}{$G_1, G_2$}
\State Compute graph Laplacians $L_1 = L(G_1)$ and $L_2 = L(G_2)$.
\State Compute $\lambda^{(1)} = \textsc{Eigenvals}(L_1)$ and  $\lambda^{(2)} = \textsc{Eigenvals}(L_2)$.
\State Initialize \textsc{Frontier} as the frontier of $\left\{ \left(M_i, \alpha_i\right)\right\}$ found by Algorithm \ref{alg:linear_alg}.
\State Set $t_c = t_\text{init}$ (some low value e.g. $10^{-3}$).
\State Compute $D^2_\text{max} = \min_i D^2(G_1, G_2 | t_c, \alpha_i, M_i)$
\While{$D^2_\text{max}$ is increasing}
    \State $t_c += \Delta t$
    \State $\textsc{Frontier} = \textsc{T-Step}(\lambda^{(1)}, \lambda^{(2)}, \textsc{Frontier}, t_c)$
    \State $D^2_\text{max} = \max_i D^2(G_1, G_2 | t_c, \alpha_i, M_i)$
\EndWhile
\State Return $\sup_t \min_i D^2(G_1, G_2 | t_c, \alpha_i, M_i )$
\EndProcedure
\end{algorithmic}

\end{algorithm}

In this section, we describe the algorithm used to calculate upper bounds on graph distances as the joint optima (over $P$, $t$, and $\alpha$) of the distance equations Equation \ref{defn:exp_defn} and Equation \ref{defn:linear_defn}, under orthogonality constraints only, i.e. the case $\mathcal{C}(P) = \{ P | P^T P = I \}$. At the core of both algorithms is a subroutine to solve the Linear Assignment Problem (LAP - see Equation \eqref{eqn:lap_defn}) repeatedly, in order to find the subpermutation matrix which is optimal at a particular value of $\alpha$. Namely, we are interested in calculating $\tilde{D}$ as  
\begin{equation}
    \begin{aligned}
    \tilde{D}(G_1, G_2) &= \min_\alpha f(\alpha) \\
    \text{where} \\
    f(\alpha) &= \inf_{P | P^T P = I} \left| \left| 
    \frac{1}{\alpha} P L(G_1) - \alpha L(G_2) P
    \right| \right| \\
    &\text{which, for orthogonality or any other compact constraint} \\
    &= \min_{P | P^T P = I} \left| \left| 
    \frac{1}{\alpha} P L(G_1) - \alpha L(G_2) P
    \right| \right|. \\
    \end{aligned}
\end{equation}

However, we have found that the unique structure of this optimization problem admits a specialized procedure which is faster and more accurate than nested univariate optimization of $\alpha$ and $t$ (where each innermost function evaluation consists of a full optimization over $P$ at some $t, \alpha$). We first briefly describe the algorithm used to find the optimal $P$ and $\alpha$ for $\tilde{D}^2$. The formal description of the algorithm is given by Algorithm \ref{alg:linear_alg}. In both cases, we reduce the computational complexity of the optimization over $P$ by imposing the additional constraint that $P$ must be a subpermutation matrix when rotated into the spectral basis (we define subpermutations in the proof of Theorem \ref{thm:LAP_bound}). This constraint is compatible with the orthogonality constraint (all subpermutation matrices are orthogonal, but not vice versa). The tradeoff of this reduction of computational complexity is that we can only guarantee that our optima are upper bounds of the optima over all orthogonal $P$. However, in practice, this bound seems to be tight: we have yet to find an example where orthogonally-constrained optimization was able to improve in objective function value over optimization constrained to subpermutation matrices. Therefore, we shall for the remainder of this paper refer to the optima calculated as distance values, when strictly they are distance upper bounds. We also note here that a distance lower bound is also possible to calculate by \emph{relaxing} the constraints in $\mathcal{C}(P)$ (for instance, by replacing the optimization over all $P$ with a less constrained matching problem - see Appendix \ref{appndx:slb}). 
\begin{figure}
    \centering
      \Large (a) \raisebox{-0.5\height}{\includegraphics[width=.65\linewidth]{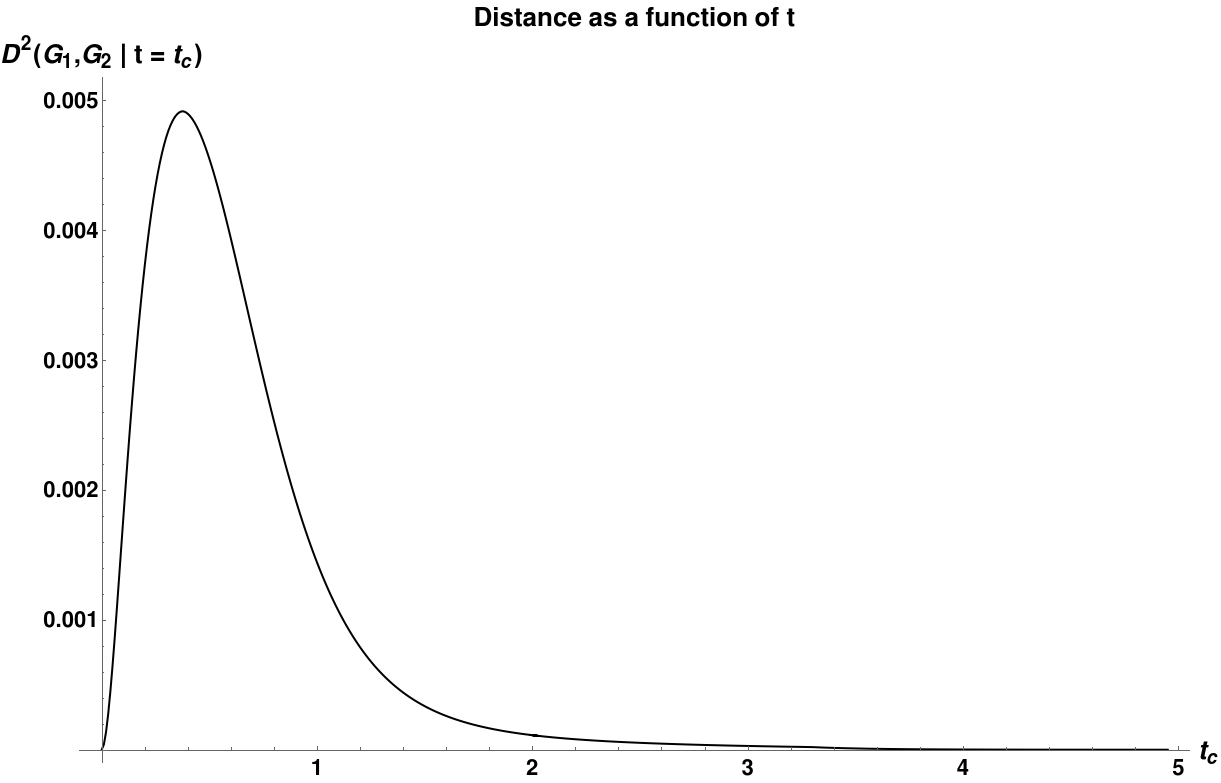}} \\
      \Large (b) \raisebox{-0.5\height}{\includegraphics[width =.65\linewidth]{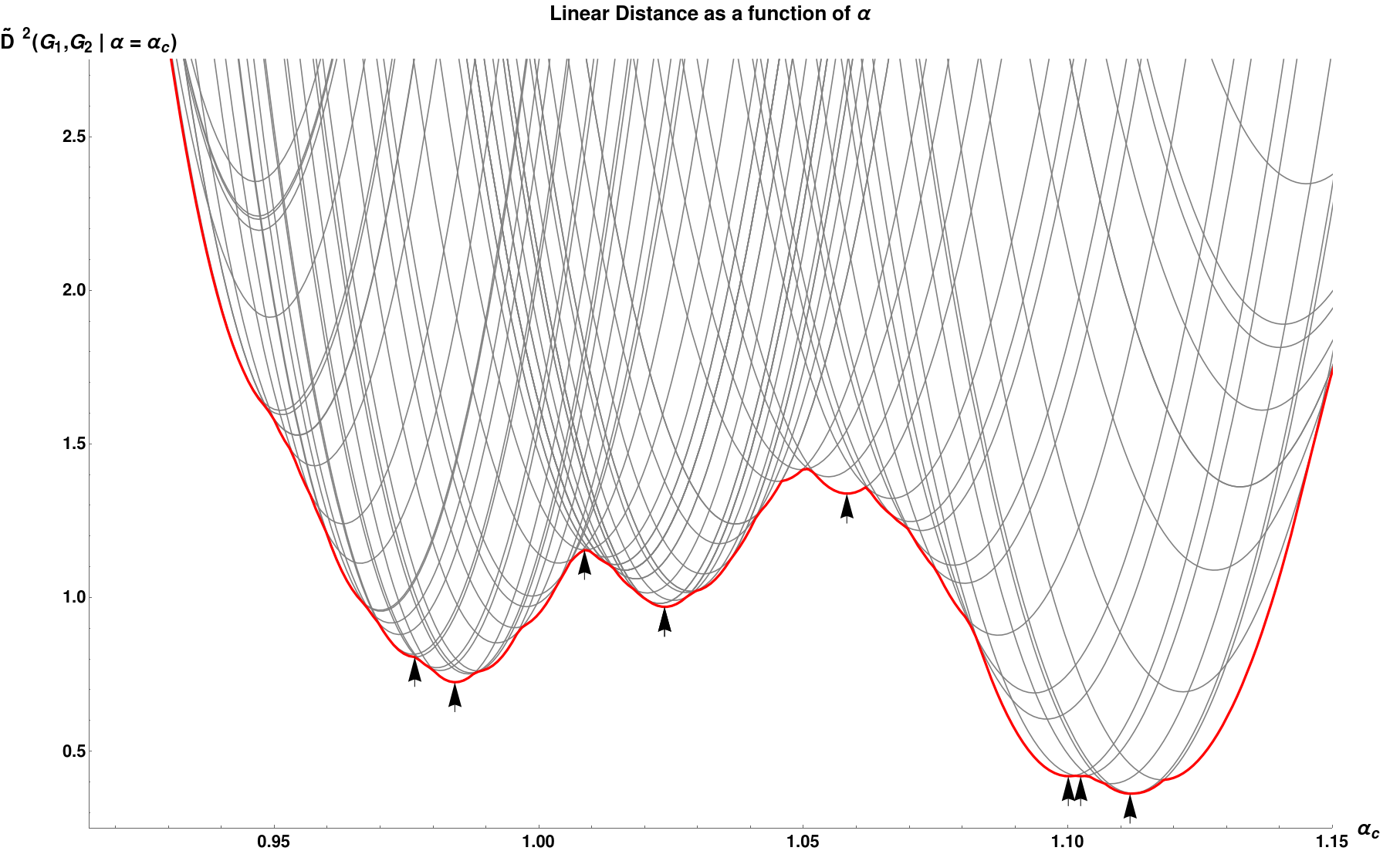}} \\
     \Large (c) \raisebox{-0.5\height}{\includegraphics[width = .65\linewidth]{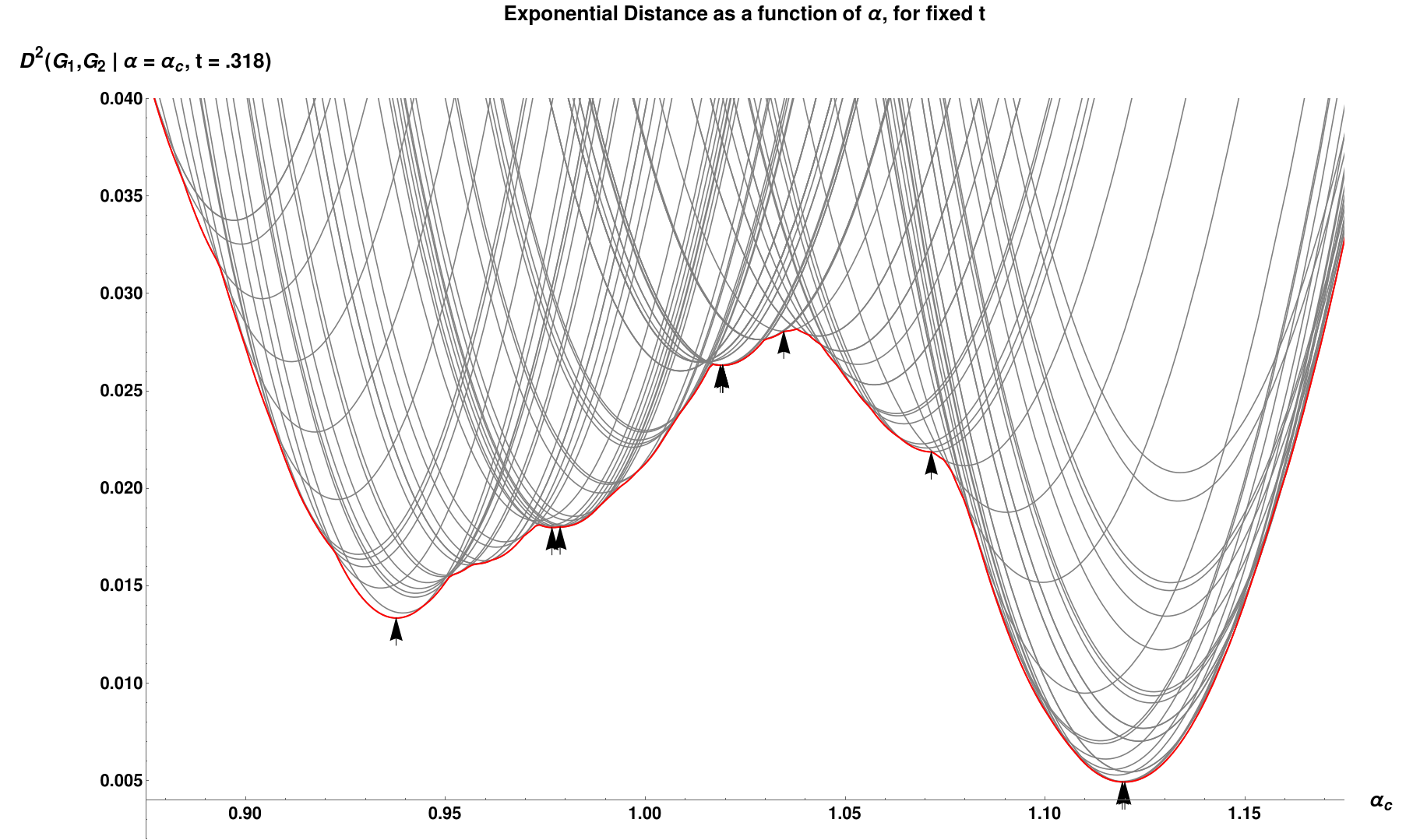}}

    \caption[lh]{Three plots demonstrating characteristics of distance calculation between a $(7 \times 7)$ grid and an $(8 \times 8)$ grid. \\ (a): a plot illustrating unimodality of $D^2 \left(\left.\text{Sq}_7, \text{Sq}_8 \right| t \right) = \inf_{\alpha > 0} \inf_{P | \mathcal{C}(P)} {\left|\left| P e^{\frac{t}{a} L(\text{Sq}_7)} - e^{t \alpha L(\text{Sq}_8)} P\right| \right|}_F^2$ as a function of $t$. The peak, at $t \approx .318$, yields the distance $D^2 \left(\text{Sq}_7, \text{Sq}_8 \right)$. \\
    (b): Plot illustrating the discontinuity and multimodality of the linear version of distance. Each gray curve represents a function $f_{P_c}(\alpha_c) = \tilde{D}^2 \left( \text{Sq}_7, \text{Sq}_8 |  \alpha_c, P_c \right)$. The thicker curve is the lower convex hull of the thinner curves as a function of $\alpha$, that is: $\boldsymbol{f}(\alpha_c) = \inf_{P|\mathcal{C}(P)} \tilde{D}^2 \left( \text{Sq}_7, \text{Sq}_8 | \alpha_c \right)$. We see that $\boldsymbol{f}(\alpha)$ is continuous, but has discontinuous slope, as well as several local optima (marked by arrowheads). These properties make $\tilde{D}$ difficult to optimize, necessitating the development of Algorithm \ref{alg:linear_alg}. \\
    (c): As in (b), but with $D^2 \left( \text{Sq}_7, \text{Sq}_8 | t = .318 \right)$ plotted instead of $\tilde{D}^2$. This $t$ value is the location of the maximum in the leftmost plot.}
    \label{fig:distance_example_fig}
\end{figure}

\subsubsection{Optimization of $\tilde{D}^2$}

Joint optimization of $\tilde{D}^2$ over $\alpha$ and $P$ is a nested optimization problem (see \cite{onoyama2000method} and \cite{sinha2018review} for a description of nested optimization), with potential combinatorial optimization over $P$ dependent on each choice of $\alpha$. Furthermore, the function $f(\alpha) = \inf_{P| \mathcal{C}(P)} \tilde{D}^2(G_1, G_2 | \alpha)$ is both multimodal and continuous but with in general discontinuous derivative (See Figure \ref{fig:distance_example_fig}). Univariate optimization procedures such as Golden Section Search result in many loops of some procedure to optimize over $P$, which in our restricted case must each time compute a full solution to a LAP with $n_2 \times n_1$ weights. In our experience, this means that these univariate methods have a tendency to get stuck in local optima. We reduce the total number of calls to the LAP solver, as well as the size of the LAPs solved, by taking advantage of several unique properties of the optimization as a function of $\alpha$. When the optimal $P^{(1)}$ and $P^{(2)}$ are known for $\alpha_1$ and $\alpha_2$, then for any $\alpha_c$ such that $\min (\alpha_1, \alpha_2) \leq \alpha_c \leq \max (\alpha_1, \alpha_2)$, the optimal $P^{(c)}$ at $\alpha_c$ must satisfy: $P^{(1)}_{ij} = 1 \wedge P^{(2)}_{ij} = 1 \implies P^{(c)}_{ij} = 1 $  (see Theorem \ref{thm:matching_agreement}).  Thus, the optimization over $P$ at $\alpha_c$ is already partially solved given the solutions at $\alpha_1$ and $\alpha_2$, and so we need only re-compute the remaining (smaller) subproblem on the set of assignments where $P^{(1)}$ and $P^{(2)}$ disagree. This has two consequences for our search over $\alpha$: First, the size of LAP problems which must be solved at each step decreases over time (as we find $P$-optima for a denser and denser set of $\alpha$). Secondly, these theoretical guarantees mean that we can mark intervals of $\alpha$-values as being explored (meaning we have provably found the $P$ which are optimal over the interval) and thus do not have to perform the relatively expensive optimization over $P$ for any $\alpha$ in that interval.

\subsubsection{Optimization of $D^2$}
Many of the theoretical guarantees underlying our algorithm for computing $\tilde{D}^2$ no longer hold for the exponential version of the distance. We adapt our linear-version procedure into an algorithm for computing this version, with the caveat that the lack of these guarantees means that our upper bound on the exponential version may be looser than that on the linear version. It is still clearly an upper bound, since the $\alpha$ and $P$ found by this procedure satisfy the given constraints $\alpha > 0$ and $P^T P = I$. In particular, we have observed cases where the exponential-distance analog of Theorem \ref{thm:matching_agreement} would not hold, meaning we cannot rule out $\alpha$-intervals as we can in the linear version. Thus, this upper bound may be looser than the that computed for the linear objective function. 

For the exponential version of the algorithm, we first compute the list of optimal $P$ for the linear version, assuming (since $e^{t L} \approx I + L$ for very small $t$) that this is also the list of optimal $P$ for the exponential version of the objective function at some low $t$. We proceed to increment $t$ with some step size $\Delta t$, in the manner of a continuation method \cite{allgower2012numerical}. At each new $t$ value, we search for new optimal $P$ along the currently known frontier of optima as a function of $\alpha$. When a new $P$ is found as the intersection of two known $P_{i}, P_{i+1}$, it is inserted into the list, which is kept in order of increasing $\alpha$. For each $P$ in this frontier, we find the optimal $\alpha$, keeping $P$ and $t$ constant. Assuming $\inf_P \inf_\alpha D^2(G_1, G_2 | t_c)$ is unimodal as a function of $t_c$, we increase $t_c$ until $\inf_P \inf_\alpha D^2(G_1, G_2 | t_c) \geq \inf_P \inf_\alpha D^2(G_1, G_2 | t_c + \Delta t)$, storing all $P$ matrices found as optima at each $t_c$ value. $P$ which were on the lower convex hull at some prior value of $t$ but not the current value are retained, as they may regain optimality for some $\alpha$-range at a future value of $t$ (we have observed this, in practice). For this list $P_1, P_2 \ldots P_m$, we then compute $\sup_t \inf_\alpha \inf_i D^2(G_1, G_2 | P_i)$. Since the exponential map is continuous, and we are incrementing $t$ by very small steps, we also propose the further computational shortcut of storing the list of optimal $\alpha$ at time $t$ to use as starting points for the optimization at $t + \Delta t$. In practice, this made little difference in the runtime of our optimization procedure. 

\subsection{Algorithm Correctness Proof}
    \label{subsec:alg_correct}
    \begin{theorem}
    \label{thm:LAP_bound}
    For any two graphs $G_1$ and $G_2$ with Laplacians $L(G_1)$ and $L(G_2)$, for fixed $\alpha$, the optimization over $P$ given in the innermost loop of Equation \ref{defn:linear_defn} is upper bounded by a Linear Assignment Problem as defined in Equation \eqref{eqn:lap_defn}. This LAP is given by taking $R$ to be the eigenvalues $\lambda^{(1)}_j$ of $L(G_1)$ and $S$ to be the eigenvalues $\lambda^{(2)}_i$ of $L(G_2)$, with the cost of a pair (equivalently, one entry of the cost matrix $C$) given by 
    \begin{align}
        \label{eqn:cost_matrix_entry}
        C_{ij} = c(s_i, r_j) = c\left(\lambda^{(2)}_i, \lambda^{(1)}_j\right) = {\left( 
        \frac{1}{\alpha} {\lambda^{(1)}_j} - 
        \alpha {\lambda^{(2)}_i} \right)}^2
    \end{align}
    \end{theorem}
    \begin{proof}
    $L(G_1)$ and $L(G_2)$ are both real symmetric matrices, so they may be diagonalized as $L(G_i) = U_i \Lambda_i U_i^T$, where the $U_i$ are rotation matrices, and the $\Lambda_i$ are diagonal matrices with the eigenvalues $\lambda^{(i)}_1, \lambda^{(i)}_2 \ldots \lambda^{(i)}_{n_i}$ along the diagonal. 
    Because the Frobenius norm is invariant under rotation, we have:
    \begin{align}
        \tilde{D}^2(G_1,G_2) &= \inf_{\alpha > 0} \inf_{P^T P = I} {\left| \left| \frac{1}{\alpha} P L(G_1) - \alpha L(G_2) P \right| \right|}^2_F \nonumber \\
        &= \inf_{\alpha > 0} \inf_{P^T P = I} {\left| \left| \frac{1}{\alpha} U_2^T P L(G_1) U_1 - \alpha U_2^T L(G_2) P U_1 \right| \right|}^2_F \nonumber \\
        &= \inf_{\alpha > 0} \inf_{P^T P = I} {\left| \left| \frac{1}{\alpha} U_2^T P U_1 \Lambda_1 U_1^T U_1 - \alpha U_2^T U_2 \Lambda_2 U_2^T P U_1 \right| \right|}^2_F \label{eqn:diagonal} \\
        &= \inf_{\alpha > 0} \inf_{P^T P = I} {\left| \left| \frac{1}{\alpha} U_2^T P U_1 \Lambda_1 - \alpha \Lambda_2 U_2^T P U_1 \right| \right|}^2_F. \nonumber \\
        \intertext{Because the $U_i$ are orthogonal, the transformation $ \tilde{P} = U_2^T P U_1$ preserves orthogonality, so}
        \tilde{D}^2(G_1,G_2) &= \inf_{\alpha > 0} \inf_{P^T P = I} {\left| \left| \frac{1}{\alpha} P \Lambda_1 - \alpha \Lambda_2 P \right| \right|}^2_F \nonumber \\
        &= \inf_{\alpha > 0} \inf_{P^T P = I} 
        {\left| \left| \frac{1}{\alpha} \Lambda_1 \right| \right|}^2_F + 
        {\left| \left| \frac{}{} \alpha \Lambda_2 P  \right| \right|}^2_F -
        2 \text{Tr} \left[ P^T \Lambda_2 P \Lambda_1 \right] \nonumber \\
        &= \inf_{\alpha > 0} \inf_{P^T P = I} \left(
        \text{Tr} \left[ \frac{1}{\alpha^2} \Lambda_1^2 \right]  + 
        \text{Tr} \left[ \alpha^2 P^T  \Lambda_2^2 P \right] -
        2 \text{Tr} \left[ P^T \Lambda_2 P \Lambda_1 \right] \right) \nonumber \\
        \intertext{writing $P = [p_{ij}]$,}
        \tilde{D}^2(G_1,G_2) &= \inf_{\alpha > 0} \inf_{P^T P = I} \left(
        \frac{1}{\alpha^2} \sum\limits_{j=1}^{n_1} {\lambda^{(1)}_j}^2  + 
        \alpha^2 \sum\limits_{i=1}^{n_2}\sum\limits_{j=1}^{n_1} p_{ij}^2 {\lambda^{(2)}_i}^2 \right. \label{eqn:f_terms} \\*
        &\left. \qquad \qquad \qquad \qquad \qquad - 2 \sum\limits_{i=1}^{n_2}\sum\limits_{j=1}^{n_1} p_{ij}^2 {\lambda^{(2)}_i} {\lambda^{(1)}_j}
        \right) \nonumber \\*
        &= \inf_{\alpha > 0} \inf_{P^T P = I} \left(
        \sum\limits_{i=1}^{n_2}\sum\limits_{j=1}^{n_1}
        p_{ij}^2 \left( 
        \frac{1}{\alpha^2} {\lambda^{(1)}_j}^2 - 
        2 {\lambda^{(2)}_i} {\lambda^{(1)}_j} +
        \alpha^2 {\lambda^{(2)}_i}^2 \right)
        \right) \nonumber \\
        &= \inf_{\alpha > 0} \inf_{P^T P = I} \left(
        \sum\limits_{i=1}^{n_2}\sum\limits_{j=1}^{n_1}
        p_{ij}^2 {\left( 
        \frac{1}{\alpha} {\lambda^{(1)}_j} - 
        \alpha {\lambda^{(2)}_i} \right)}^2
        \right) \label{eqn:simple_d_tilde}
    \end{align}
    For any given $\alpha$, 
    \begin{align}
     \inf_{P^T P = I} \left(
        \sum\limits_{i=1}^{n_2}\sum\limits_{j=1}^{n_1}
        p_{ij}^2 {\left( 
        \frac{1}{\alpha} {\lambda^{(1)}_j} - 
        \alpha {\lambda^{(2)}_i} \right)}^2
        \right) \\
        \qquad \leq
    \inf_{\tilde{P} | \text{subperm}(\tilde{P})} \left(
        \sum\limits_{i=1}^{n_2}\sum\limits_{j=1}^{n_1}
        \tilde{p}_{ij}^2 {\left( 
        \frac{1}{\alpha} {\lambda^{(1)}_j} - 
        \alpha {\lambda^{(2)}_i} \right)}^2
        \right) \quad , \nonumber
    \end{align}
    
    \noindent where $\text{subperm}(\tilde{P})$ could be any other condition more strict than the constraint $P^T P = I$. Here we take this stricter constraint to be 
    the condition that $\tilde{P}$ is a {\it subpermutation matrix}: an orthogonal matrix (i.e.  $\tilde{P}^T \tilde{P}=I)$ 
    for which $\tilde{P} \in \{0,1\}^{n_2 \times n_1}$. 
    Equivalently, a subpermutation matrix is a $\{0,1\}$-valued matrix $[\tilde{p}_{ij}] $ such that for each $i \in \{1, \ldots n_1 \leq n_2 \}$,
    exactly one $j \in \{1, \ldots n_2 \geq n_1\}$ takes the value 1 rather than 0 (so $\sum_{j=1}^{n_2} \tilde{P}_{j i}=1$), and 
    for each $j \in \{1, \ldots n_2 \geq n_1 \}$,
    either zero or one $i \in \{1, \ldots n_1 \leq n_2\}$ takes the value 1 rather than 0
    (so $\sum_{i=1}^{n_1} \tilde{P}_{j i} \leq 1$).

     Furthermore, this optimization is exactly a linear assignment problem of eigenvalues of $L(G_1)$ to $L(G_2)$, with the cost of a pair $\left(\lambda^{(1)}_j, \lambda^{(2)}_i \right)$ given by 
    \[
        c\left(\lambda^{(1)}_j, \lambda^{(2)}_i \right) =
        {\left( 
        \frac{1}{\alpha} {\lambda^{(1)}_j} - 
        \alpha {\lambda^{(2)}_i} \right)}^2
         \]
         
    Note also that the same argument applies to the innermost two optimizations of the calculation of $D^2$ (the exponential version of the diffusion distance) as well as $D^2_r$. In the $D^2$ case the entries of the cost matrix are instead given by 
    \[
        c\left(\lambda^{(1)}_j, \lambda^{(2)}_i \right) =
        {\left( 
        e^{\frac{1}{\alpha} {\lambda^{(1)}_j}} - 
        e^{\alpha {\lambda^{(2)}_i}} \right)}^2
         \]
     If we instead loosen the constraints on $P$, we can calculate a lower bound on the distance. See Appendix \ref{appndx:slb} for lower bound details. 
    
    \end{proof}
    Recall that our definition of a `matching' in Section \ref{subsec:defns} was a $P$ matrix representing a particular solution to the linear assignment problem with costs given as in Equation \eqref{eqn:cost_matrix_entry}. For given $G_1, G_2$, and some matching $M$, let 
    \begin{align}
        \label{eqn:f_alpha_defn}
        f_M(\alpha) = \tilde{D}^2 (G_1, G_2 | \alpha, U_2^T M U_1)
    \end{align}
    where $U_1, U_2$ diagonalize $L_1$ and $L_2$ as in Equation \eqref{eqn:diagonal}. 
    \begin{lemma}
    \label{lem:unique_soln}
    For two unique matchings $M_1$ and $M_2$ (for the same $G_1, G_2$) the equation $f_{M_1}(\alpha) - f_{M_2}(\alpha) = 0$ has at most one real positive solution in $\alpha$. This follows from the fact that when $P$ and $t$ are fixed, the objective function is a rational function in $\alpha$ (see Equation \eqref{eqn:f_terms}), with a quadratic numerator and an asymptote at $\alpha = 0$.
    \end{lemma}
    \begin{proof}
    By Equation \eqref{eqn:f_terms}, we have  
    \begin{align}
        & f_{M_1}(\alpha) - f_{M_2}(\alpha) = \nonumber \\
        & \left(
        \frac{1}{\alpha^2} \sum\limits_{j=1}^{n_1} {\lambda^{(1)}_j}^2  + 
        \alpha^2 \sum\limits_{i=1}^{n_2}\sum\limits_{j=1}^{n_1} {[M_1]}_{ij}^2 {\lambda^{(2)}_i}^2 -
        2 \sum\limits_{i=1}^{n_2}\sum\limits_{j=1}^{n_1} {[M_1]}_{ij}^2 {\lambda^{(2)}_i} {\lambda^{(1)}_j}
        \right) \\
        & \quad - \left(
        \frac{1}{\alpha^2} \sum\limits_{j=1}^{n_1} {\lambda^{(1)}_j}^2  + 
        \alpha^2 \sum\limits_{i=1}^{n_2}\sum\limits_{j=1}^{n_1} {[M_2]}_{ij}^2 {\lambda^{(2)}_i}^2 -
        2 \sum\limits_{i=1}^{n_2}\sum\limits_{j=1}^{n_1} {[M_2]}_{ij}^2 {\lambda^{(2)}_i} {\lambda^{(1)}_j}
        \right) \\
        &= 
        \alpha^2 \left( \sum\limits_{i=1}^{n_2}\sum\limits_{j=1}^{n_1} {[M_1]}_{ij}^2 {\lambda^{(2)}_i}^2
        - \sum\limits_{i=1}^{n_2}\sum\limits_{j=1}^{n_1} {[M_2]}_{ij}^2 {\lambda^{(2)}_i}^2 \right) \\
        & \quad + \left(
        2 \sum\limits_{i=1}^{n_2}\sum\limits_{j=1}^{n_1} {[M_2]}_{ij}^2 {\lambda^{(2)}_i} {\lambda^{(1)}_j}
        -
        2 \sum\limits_{i=1}^{n_2}\sum\limits_{j=1}^{n_1} {[M_1]}_{ij}^2 {\lambda^{(2)}_i} {\lambda^{(1)}_j}
        \right)
        \end{align}
        \begin{align}
        \intertext{Abbreviating the sums, we have}
        \alpha^2 \left(A_1 - A_2\right) + \left(C_2 - C_1\right) &= 0 \\
        \intertext{and so}
        \alpha = \pm \sqrt{\frac{C_2 - C_1}{A_1 - A_2}}
    \end{align}
    Since $A_1, A_2, C_1, C_2$ are all nonnegative reals, at most one of these roots is positive.
    \end{proof}
    We will say that a matching $M$ ``assigns'' $j$ to $i$ if and only if $M_{ij} = 1$.
    \begin{theorem}
    \label{thm:matching_agreement}
    If two matchings $M_1$ and $M_3$ which yield optimal upper bounds for the linear distance $\tilde{D}^2$ (at $\alpha_1 \leq \alpha$ and $\alpha_3 \geq \alpha$ respectively) agree on a set of assignments, then the optimal $M$ at $\alpha$ must also agree with that set of assignments.
    \end{theorem}
    \begin{proof}
    We need the following lemmas:
    \begin{lemma}
    \label{lem:mono_incr}
    If an optimal matching assigns $i$ to $m(i)$ (so that eigenvalue $\lambda^{(1)}_i$ of $G_1$ is paired with $\lambda^{(2)}_{f(i)}$ of $G_2$ in the sum of costs Equation \eqref{eqn:cost_matrix_entry}), then the sequence $m(1), m(2), \ldots m(n_1)$ is monotonic increasing.
    \end{lemma}
    \begin{proof}
    
    This follows from the fact that the two sequences of eigenvalues are monotonic nondecreasing, so if there's a `crossing' ($i_1 < i_2$ but $m(i_2) < m(i_1)$) then the new matching obtained by uncrossing those two pairs (performing a 2-opt step as defined in \cite{croes1958method}) has strictly lesser objective function value. Hence an optimal matching can't contain any such crossings.
    \end{proof}

    \begin{lemma}
    For all positive real $\alpha^* \geq \epsilon > 0$, let $M_1$ be an optimal matching at $\alpha^* - \epsilon$ and $M_2$ be optimal at $\alpha^* + \epsilon$. For $1 \leq i \leq n_1$, let $s_1(i)$ and $s_2(i)$ be the indices of $\lambda^{(2)}$ paired with $i$ in $M_1$ and $M_2$, respectively. Then for all $i$,  $s_1(i) \leq s_2(i)$.
    \end{lemma}
    \begin{proof}
    Define a ``run'' for $s_1, s_2$ as a sequence of consecutive indices $l, l+1, \ldots l+k$ in $[1,n_1]$ such that for any $l$, $l+1$: $\min(s_1(l+1), s_2(l+1)) < \max(s_1(l), s_2(l))$. The following must be true about a ``run'': 
    \begin{enumerate}
        \item Within a run, either $s_1(l) < s_2(l)$ or $s_1(l) > s_2(l)$ for all $l$. Otherwise, we have one or more crossings (as in Lemma \ref{lem:mono_incr}): for some $l$ we have $s_1(l) > s_1(l+1)$ or $s_2(l) > s_2(l+1)$. Any crossing may be uncrossed for a strictly lower objective function value - violating optimality of $M_1$ or $M_2$.  
        \item Any pair of matchings as defined above consists of a sequence of runs, where we allow a run to be trivial i.e. be a single index.
    \end{enumerate}
    Next, we show that within a run, we must have $s_1(i) < s_2(i)$ for all $i$. Let $S = \{l, l+1, \ldots l+k \}$ be a run. By optimality of $M_1$, $M_2$ at $\alpha^* - \epsilon$ and $\alpha^* + \epsilon$ respectively, we have:
    
    \begin{align*}
    \label{eqn:dist_comp_sum}
      \sum_{i \in S} {\left(\frac{1}{\alpha^* - \epsilon}\lambda^{(1)}_i - (\alpha^*-\epsilon)\lambda^{(2)}_{s_1(i)} \right)}^2 < \sum_{i\in S} {\left( \frac{1}{\alpha^* - \epsilon}\lambda^{(1)}_i - (\alpha^*-\epsilon)\lambda^{(2)}_{s_2(i)} \right)}^2 \\ \intertext{and}
      \sum_{i \in S} {\left(\frac{1}{\alpha^* + \epsilon}\lambda^{(1)}_i - (\alpha^*+\epsilon)\lambda^{(2)}_{s_2(i)}\right)}^2 < \sum_{i \in S} {\left(\frac{1}{\alpha + \epsilon}\lambda^{(1)}_i - (\alpha+\epsilon)\lambda^{(2)}_{s_1(i)}\right)}^2.
    \end{align*}
    Respectively, these simplify to
    \begin{align*}  
      - \sum_{i \in S}\left( \lambda^{(2)}_{s_1(i)} - \lambda^{(2)}_{s_2(i)}  \right) \left( -2 \lambda^{(i)}_{i} + {(\alpha^* - \epsilon)}^2 \left( \lambda^{(2)}_{s_1(i)} + \lambda^{(2)}_{s_2(i)}  \right) \right) > 0 \\
      \intertext{and}
      \sum_{i \in S} \left( \lambda^{(2)}_{s_1(i)} - \lambda^{(2)}_{s_2(i)}  \right) \left( -2 \lambda^{(i)}_{i} + {(\alpha^* + \epsilon)}^2 \left( \lambda^{(2)}_{s_1(i)} + \lambda^{(2)}_{s_2(i)}  \right) \right) > 0. \\
      \end{align*}
      Summing these inequalities and cancelling $-2 \lambda^{(i)}_{i}$, we have:
      \begin{align*}
      &\sum_{i \in S} \left\{ {(\alpha^* + \epsilon)}^2 \left( {\left(\lambda^{(2)}_{s_1(i)}\right)}^2 + {\left(\lambda^{(2)}_{s_2(i)}\right)}^2  \right) \right. \\
      & \quad \quad \quad - \left. {(\alpha^* - \epsilon)}^2 \left( {\left(\lambda^{(2)}_{s_1(i)}\right)}^2 + {\left(\lambda^{(2)}_{s_2(i)}\right)}^2  \right) \right\}> 0.
    \end{align*}
    Summing and reducing gives us    
    \[4 \alpha^* \epsilon \left( \sum_{i \in S} {\left(\lambda^{(2)}_{s_1(i)}\right)}^2 - \sum_{i \in S} {\left( \lambda^{(2)}_{s_2(i)} \right) }^2 \right) > 0 \]
    and so 
    \[\sum_{i \in S} {\left(\lambda^{(2)}_{s_1(i)}\right)}^2  > 
    \sum_{i \in S} {\left( \lambda^{(2)}_{s_2(i)} \right) }^2. \]
    
    However, since the $\lambda^{(2)}_j$ are monotonic nondecreasing, this means we cannot also have $s_1(i) > s_2(i)$ for all $i \in S$, since that would imply  
    \[\sum_{i=1}^{n_1} {\left(\lambda^{(2)}_{s_1(i)}\right)}^2  <
    \sum_{i=1}^{n_1} {\left( \lambda^{(2)}_{s_2(i)} \right) }^2. \]
    
    Therefore, in a run of arbitrary length, all indices must move `forward' (meaning that $s_1(i) < s_2(i)$ for all $i$ in the run), and so (since any pair of matchings optimal at such $\alpha$ define a set of runs) we must have $s_1(i) \leq s_2(i)$. This completes the proof of the lemma.

    \end{proof}
    Thus, for three matchings $M_1, M_2, M_3$ which optimal at a sequence of $\alpha_1 \leq \alpha_2 \leq \alpha_3$, we must have $s_1(i) \leq s_2(i) \leq s_3(i)$ for all $i$. In particular, if $s_1(i) = s_3(i)$, we must also have $s_1(i) = s_2(i) = s_3(i)$.
    \end{proof}
    
    \begin{theorem}If two matchings $M_1$ and $M_3$ yield optimal upper bounds for the linear distance $\tilde{D}^2$ at $\alpha_1$ and $\alpha_3$ respectively, and $f_{M_1}(\alpha_2) = f_{M_2}(\alpha_2)$ for some $\alpha_2$ s.t. $\alpha_1 \leq \alpha_2 \leq \alpha_3$, then either (1) $M_1$ and $M_3$ are optimal over the entire interval $[\alpha_1, \alpha_3]$ or (1) some other matching $M_2$ improves over $M_1$ and $M_3$ at $\alpha_2$.
    \end{theorem}
    \begin{proof}
    This follows directly from the facts that $f_{M_1}(\alpha)$  and $f_{M_2}(\alpha)$ (as defined in Equation \eqref{eqn:f_alpha_defn}), can only meet at one real positive value of $\alpha$ (Lemma \ref{lem:unique_soln}). Say that the cost curves for $M_1$ (known to be optimal at $\alpha = \alpha_1$) and $M_3$ (optimal at $\alpha = \alpha_3$) meet at $\alpha = \alpha_2$, and furthermore assume that $\alpha_1 \leq \alpha_2 \leq \alpha_3$. If some other matching $M_2$ improves over (meaning, has lesser obj. function value as a function of $\alpha$) $M_1$ or $M_3$ anywhere in the interval $[\alpha_1, \alpha_3]$, it must improve over both at $\alpha=\alpha_2$, since it may intersect each of these cost curves at most once on this interval. If $M_1$ and $M_3$ are both optimal at their intersection point (meaning no such distinct $M_2$ exists) then we know that no other matching improves on either of them over the the interval $[\alpha_1, \alpha_3]$ and may therefore mark it as explored during the outermost loop (otimization over $\alpha$) of Algorithm \ref{alg:linear_alg}.
    \end{proof}
Together, the preceeding properties verify that our algorithm will indeed find the joint optimum over all $\alpha$ and $P$ (for fixed $t=c$, for $\tilde{D}$, subject to subpermutation constraints on $P$): it allows us to find the entire set of $P$ subpermutation matrices which appear on the lower convex hull of distance as a function of alpha.

\subsection{Implementation Details}
 We implement Algorithms \ref{alg:linear_alg} and \ref{alg:exponential_alg} in the programming language ``Python'' (version 3.6.1) \cite{rossum1995python}. Numerical arrays were stored using the \emph{numpy} package \cite{van2011numpy}. Our inner LAP solver was the package \emph{lapsolver} \cite{heindl2018lapsolver}. Univariate optimization over $t$ and $\alpha$ was performed with the `bounded' method of the \emph{scipy.optimize} package \cite{scipy}, with bounds set at $[0, 10.0]$ for each variable and a input tolerance of $10^{-12}$. Laplacians were computed with the \textit{laplacian} method from the package \emph{networkX} \cite{hagberg2008exploring}, and their eigenvalues were computed with \textit{scipy.linalg.eigh}.
 
 Because of numerical precision issues arising during eigenvalue computation, it can be difficult to determine when two matchings agree, using eigenvalue comparison. In practice we ignore this issue and assume that two matchings are only identical if they associate the same indices of the two lists of eigenvalues. This means we may be accumulating multiple equivalent representations of the same matching (up to multiplicity of eigenvalues) during our sweeps through $t$ and $\alpha$. We leave mitigating this inefficiency for future work. 
\section{Numerical Experiments}
\subsection{Graph Lineages}
\label{sec:graph_families}
In this subsection we introduce several graph lineages for which we will compute various intra- and inter-lineage distances. Three of these are well-known lineages of graphs, and the fourth is defined in terms of a product of complete graphs:
\begin{itemize}
    \item Path Graphs ($\text{Pa}_n$): 1D grid graphs of length $n$, with aperiodic boundary conditions. 
    \item Cycle Graphs ($\text{Cy}_n$): 1D grid graphs of length $n$, with periodic boundary conditions.
    \item Square Grid Graphs ($\text{Sq}_n$): 2D grid graphs of dimensions $n$, with aperiodic boundary conditions. $\text{Sq}_n = \text{Pa}_n \Box \text{Pa}_n$
    \item ``Multi-Barbell'' Graphs ($\text{Ba}_n$): Constructed as $\text{Cy}_n \Box K_n$, where $K_n$ is the complete graph on $n$ vertices.  
\end{itemize}
Examples of each of these graph lineages can be seen in Figure \ref{fig:graph_families}.

\begin{figure}
 {\centering
    2D Grids \\
    \includegraphics[width=.75\linewidth]{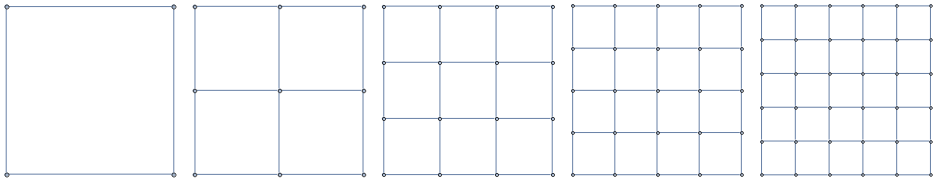} \\
     Paths \\
    \includegraphics[width=.75\linewidth]{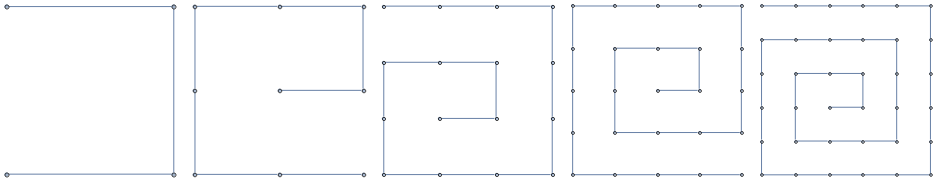} \\
    Cycles \\
    \includegraphics[width=.75\linewidth]{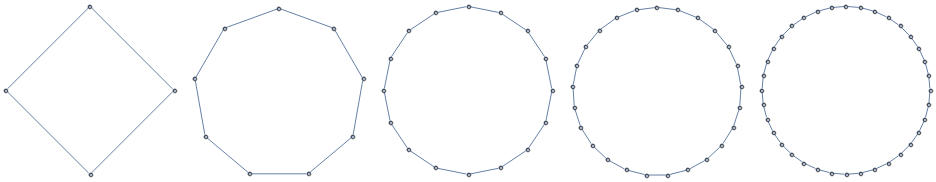} \\
    k-Barbell graphs \\
    \includegraphics[width=.75\linewidth]{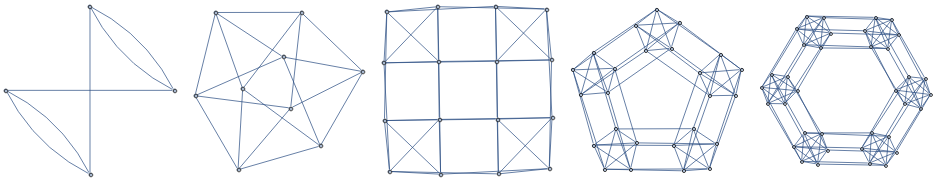} \\
}
    \caption{Graph lineages used in multiple numerical experiments in this section. }
    \label{fig:graph_families}
\end{figure}
\subsection{Numerical Optimization Methods}
We briefly discuss here the other numerical methods we have used to calculate $\tilde{D}^2$ and $D^2$. 
\begin{itemize}
    \item Nelder-Mead in Mathematica: For very small graph pairs ($n_1 \times n_2 \leq 100$) we are able to find optimal $P, \alpha, t$ using constrained optimization in Mathematica 11.3 \cite{wolfram2018mathematica} using NMinimize, which uses Nelder-Mead as its backend by default. The size limitation made this approach unusable for any real experiments.   
    \item We also tried a variety of codes specialized for numeric optimization subject to orthogonality constraints. These included (1) the python package PyManopt \cite{townsend2016pymanopt}, a code designed for manifold-constrained optimization; (2) gradient descent in Tensorflow using the penalty function $g(P) = c {\left| \left| P^T P - I \right| \right|}_F$ (with $c \ll 1$ a small positive constant weight) to maintain orthogonality, as well as (3) an implementation of the Cayley reparametrization method from \cite{wen2013feasible} (written by the authors of that same paper). In our experience, these codes were slower, with poorer scaling with problem size, than combinatorial optimization over subpermutation matrices, and did not produce improved results on our optimization problem.
    \item We compare in more detail two methods of joint optimization over $\alpha$ and $P$ when $P$ is constrained to be a subpermutation matrix in the diagonal basis for $L(G_1)$ and $L(G_2)$. Specifically, we compare our approach given in Algorithm \ref{alg:linear_alg} to univariate optimization over $\alpha$, where each function evaluation consists of full optimization over $P$. Figure \ref{fig:alg_timing} shows the results of this experiment. We randomly sample pairs of graphs as follows: 
    \begin{enumerate}
        \item $n_1$ is drawn uniformly from $[ 5, 120]$.
        \item $n_2$ is drawn uniformly from $[n_1, n_1 + 60 ] $.
        \item $G_1$ and $G_2$ are generated by adding edges according to a Bernoulli distribution with probability $p$. We ran 60 trials for each $p$ in \{ .125, .25, .375, .5, .625, .75, .875 \}.
    \end{enumerate}
    We compute the linear version of distance for each pair. Because our algorithm finds all of the local minima as a function of alpha, we compute the cost of the golden section approach as the summed cost of multiple golden section searches in alpha: one GS search starting from the initial bracket $[0.618 \alpha^*, 1.618 \alpha^*]$ for each local minimum $\alpha^*$ found by our algorithm. We see that our algorithm is always faster by at least a factor of 10, and occasionally faster by as much as a factor of $10^3$. This can be attributed to the fact that the golden section search is unaware of the structure of the linear assignment problem: it must solve a full $n_2 \times n_2$ linear assignment problem for each value of $\alpha$ it explores. In contrast, our algorithm is able to use information from prior calls to the LAP solver, and therefore solves a series of LAP problems whose sizes are monotonically nonincreasing.

\end{itemize}
\begin{figure}
    \centering
    \includegraphics[width=.6\linewidth]{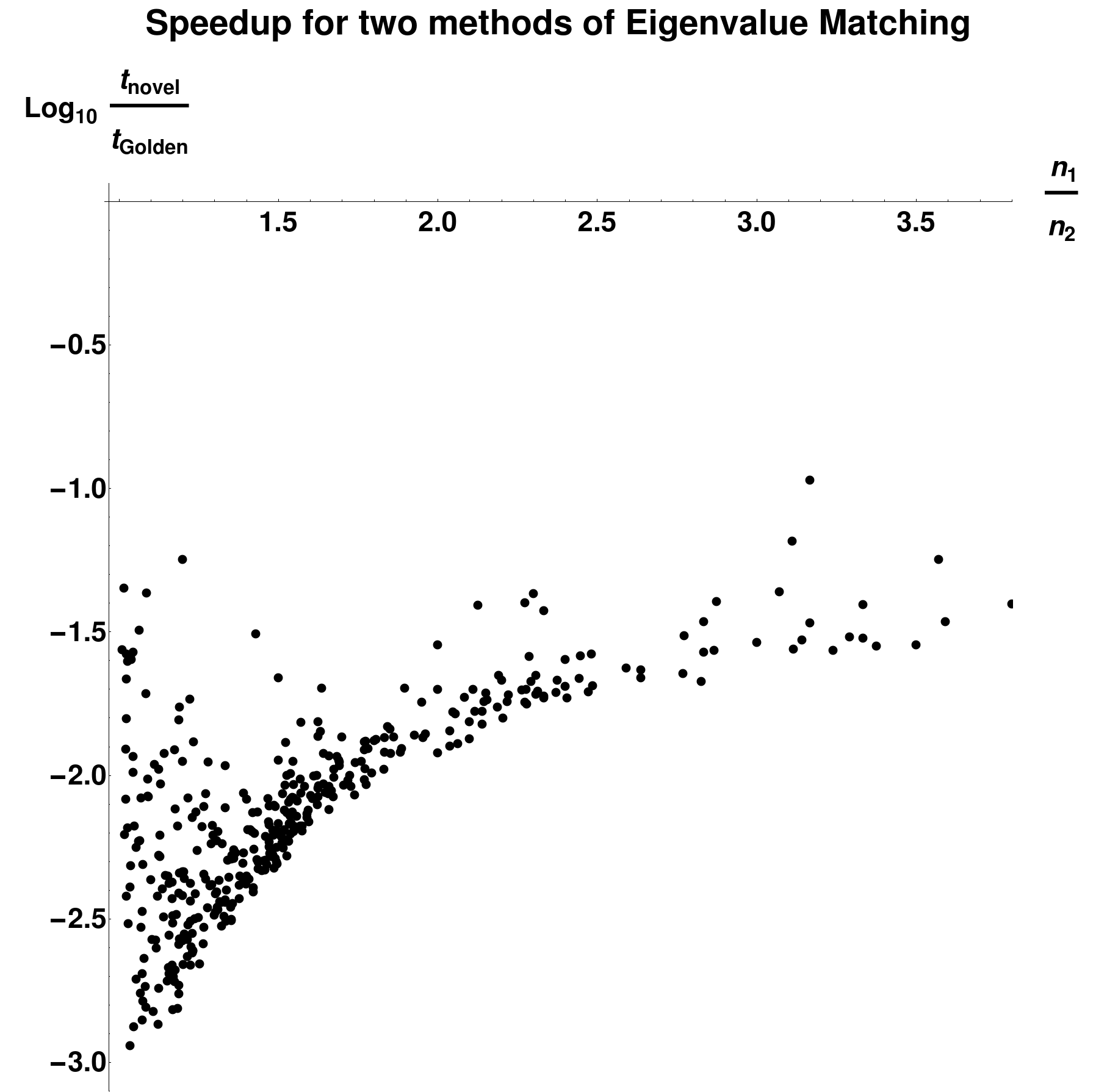}
    \caption{Comparison of runtimes for our algorithm and bounded golden section search over the same interval $[ 10^{-6}, 10]$. Runtimes were measured by a weighted count of evaluations of the Linear Assignment Problem solver, with an $n \times n$ linear assignment problem counted as $n^3$ units of cost. Because our algorithm recovers the entire lower convex hull of the objective function as a function of $\alpha$, we compute the cost of the golden section search as the summed cost of multiple searches, starting from an interval bracketing each local optimum found by our algorithm. We see that our algorithm is much less computationally expensive, sometimes by a factor of $10^{3}$. The most dramatic speedup occurs in the regime where $n_1 << n_2$. Graphs were generated by drawing $n_1$ uniformly from $[5,120]$, drawing $n_2$ uniformly from $[n_1, n_1 + 60]$, and then adding edges according to a Bernoulli distribution with $p$ in \{ .125, .25, .375, .5, .625, .75, .875 \} (60 trials each).}
    \label{fig:alg_timing}
\end{figure}
\subsection{Experiments}
\subsubsection{Triangle Inequality violation of $D$ (Exponential Distance) and $\tilde{D}$ (Linear Distance).}
\label{subsub:tri_ineq_viol}

As stated in Section \ref{subsec:dist_variants}, our full graph dissimilarity measure does not necessarily obey the triangle inequality. In this section we systematically explore conditions under which the triangle inequality is satisfied or not satisfied. We generate triplets $G_1, G_2, G_3$ of random graphs of sizes $n_i$ for $n_1 \in [5, 30]$, $n_2 \in [n_1, n_1 + 30]$, and $n_3 \in [n_2, n_2 + 30]$ by drawing edges from the Bernoulli distribution with probability $p$ (we perform 4500 trials for each $p$ value in [.125, .25, .375, .5, .625, .75, .875]). We compute the distance $\tilde{D}(G_i, G_k)$  (for $(i,k) \in \{ (1,3), (1,2), (2,3) \}$). The results may be seen in Figure \ref{fig:triangle_ineq_histogram}. In this figure we plot a histogram of the ``discrepancy score'' 
\begin{align}
    \text{Disc}(G_1, G_2, G_3) = \tilde{D}(G_1, G_3)/(\tilde{D}(G_1, G_2) + \tilde{D}(G_2, G_3)),
\end{align}
which measures the degree to which a triplet of graphs violates the triangle inequality (i.e. falls outside of the unit interval [0,1]), for approximately $3 \times 10^4$ such triplets. It is clear that, especially for the linear definition of the distance, the triangle inequality is not always satisfied. However, we also observe that (for graphs of these sizes) the discrepancy score is bounded: no triple violates the triangle inequality by more than a factor of approximately 1.8. This is shown by the histogram of discrepancies in Figure \ref{fig:triangle_ineq_histogram}. Additionally, the triangle inequality is satisfied in 28184 (95.2\%) of cases.

We see similar but even stronger results when we run the same experiment with $D^2$ instead of $\tilde{D}^2$; these may also be seen in Figure \ref{fig:triangle_ineq_histogram}. We calculated the discrepancy score analogously, but with $D$ substituted for $\tilde{D}$. We see similarly that the degree of violation is bounded. In this case, no triple violated the triangle inequality by a factor of more than 5, and in this case the triangle inequality was satisfied in 99.8\% of the triples. 
In both of these cases, the triangle inequality violations may be a result of our optimization procedure finding local minima/maxima for one or more of the three distances computed. We also repeat the above procedure for the same triplets of graphs, but with distances computed not in order of increasing vertex size:  calculating $\text{Disc}(G_2, G_1, G_3)$ and $\text{Disc}(G_3, G_2, G_1)$. All of these results are plotted in Figure \ref{fig:triangle_ineq_histogram}. 
\begin{figure}
    {\centering
    \null \hfil \includegraphics[width=.9\linewidth]{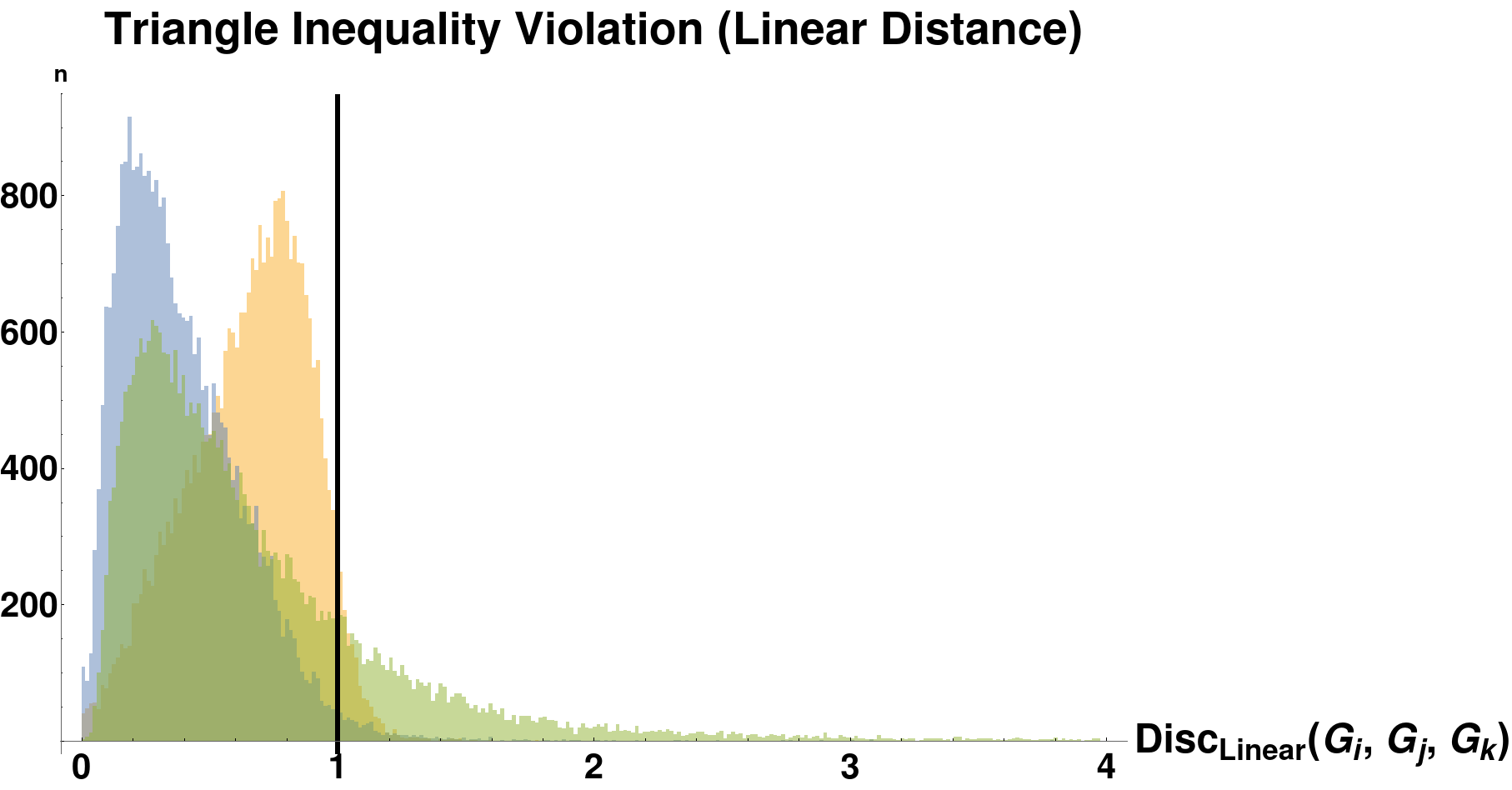} \hfil \null \\
    \null \hfil \includegraphics[width=.9\linewidth]{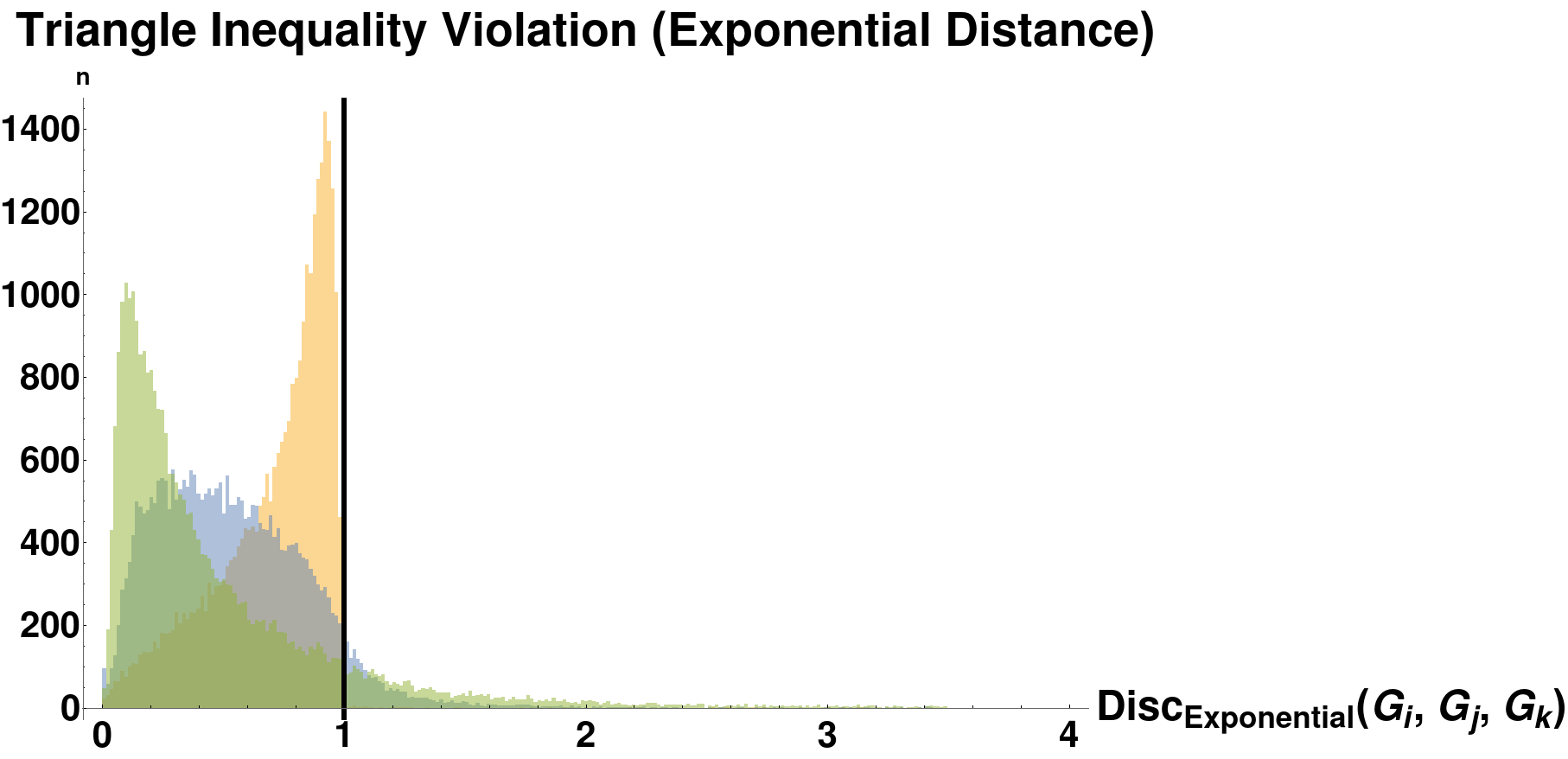} \hfil \null \\
    \hfill \includegraphics[width=.65\linewidth]{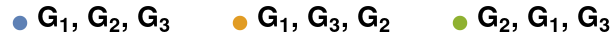} \hfill
    }
    \caption{Histograms of triangle inequality violation. These plots show the distribution of  $\text{Disc}(G_1, G_2, G_3)$, as defined in the text, for the cases (a) top: the linear or small-time version of distance and (b) bottom: the exponential or arbitrary-time version of distance. We see that for the sizes of graph we consider, the largest violation of the triangle inequality is bounded, suggesting that our distance measure may be an infra-$\rho$-pseudometric for some value of $\rho \approx 1.8$ (linear version) or $\rho \approx 5.0$ (exponential version). See Table \ref{tab:dist_versions} for a summary of the distance metric variants introduced in this paper. We also plot the same histogram for out-of-order (by vertex size) graph sequences: $\text{Disc}(G_2, G_1, G_3)$ and $\text{Disc}(G_3, G_2, G_1)$. Each plot has a line at $x=1$, the maximum discrepancy score for which the underlying distances satisfy the triangle inequality.}
    \label{fig:triangle_ineq_histogram}
\end{figure}
\vspace{-5pt}
\subsubsection{Intra- and Inter-Lineage Distances}
\label{subsub:lineage_dist}
We compute pairwise distances for sequences of graphs in the graph lineages displayed in Figure \ref{fig:graph_families}. For each pair of graph families (Square Grids, Paths, Cycles, and Multi-Barbells), we compute the distance from the $i$th member of one lineage to the $(i+1)$-st member of each other lineage, and take the average of the resulting distances from $i=1$ to $i=12$. These distances are listed in Table \ref{tab:dist_table}. As expected, average distances within a lineage are smaller than the distances from one lineage to another. 
\begin{table}[H]
    \centering
    \begin{tabular}{|c|l|l|l|l|}
        \hline
         & Square Grids & Paths & Cycles & Multi-Barbells\\ 
        \hline
        Square Grids & 0.0096700 & 0.048162 & 0.046841 & 0.63429\\ 
        \hline
        Paths & 0.30256 & 0.0018735 & 0.010300 & 2.1483\\ 
        \hline
        Cycles & 0.27150 & 0.0083606 & 0.0060738 & 2.0357\\ 
        \hline
        Multi-Barbells & 0.21666 & 0.75212 & 0.72697 & 0.029317 \\
        \hline
    \end{tabular}
    \caption{Mean distances between graphs in several lineages. For two lineages $G_1, G_2 \ldots$ (listed at left) and $H_!, H_2, \ldots$ (listed at the top), each entry shows the mean distance $D(G_i, H_{i+1})$ (where the average is taken over $i=1$ to $12$). As expected, we see that the distance from elements of a graph lineage to other members of the same lineage (the diagonal entries of the table) is smaller than distances taken between lineages. Furthermore as expected, 1D paths are more similar (but not equal) to 1D cycles than to other graph lineages. }
    \label{tab:dist_table}
\end{table}
\begin{figure}
    {\centering
    \null \hfil \includegraphics[width=.75\linewidth]{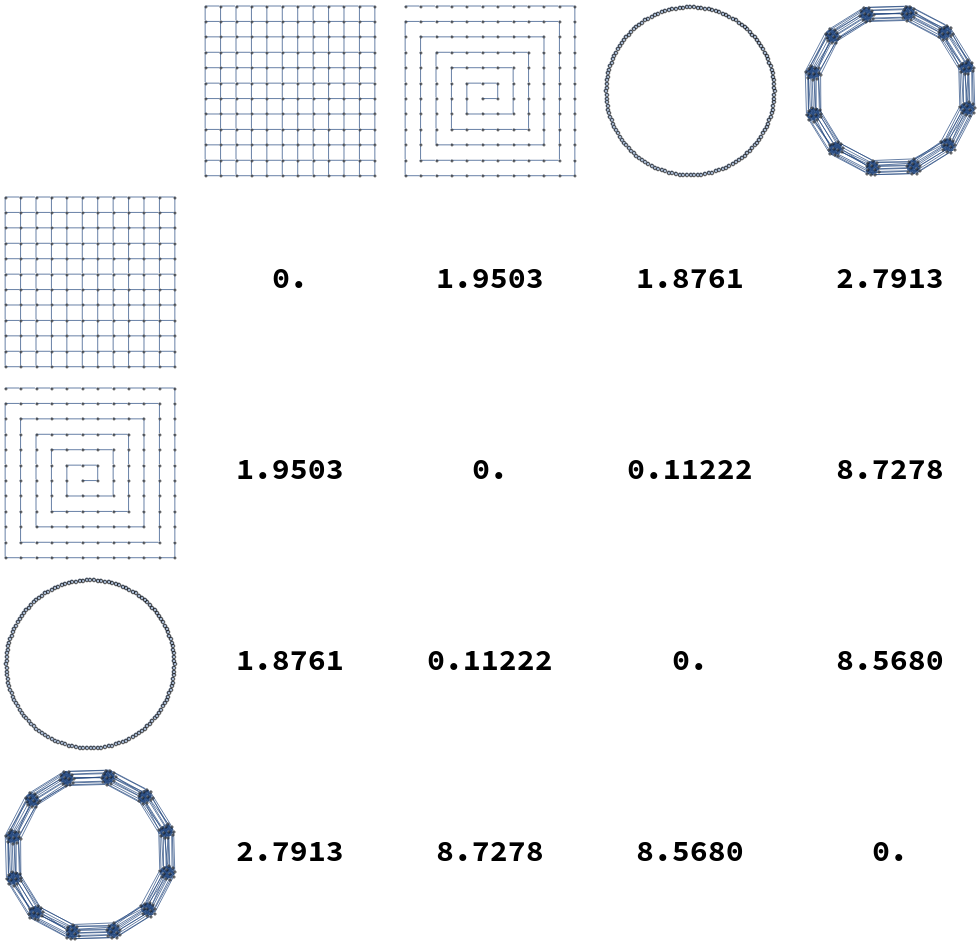}\hfil \null\\
    \null \hfil \includegraphics[width=.75\linewidth]{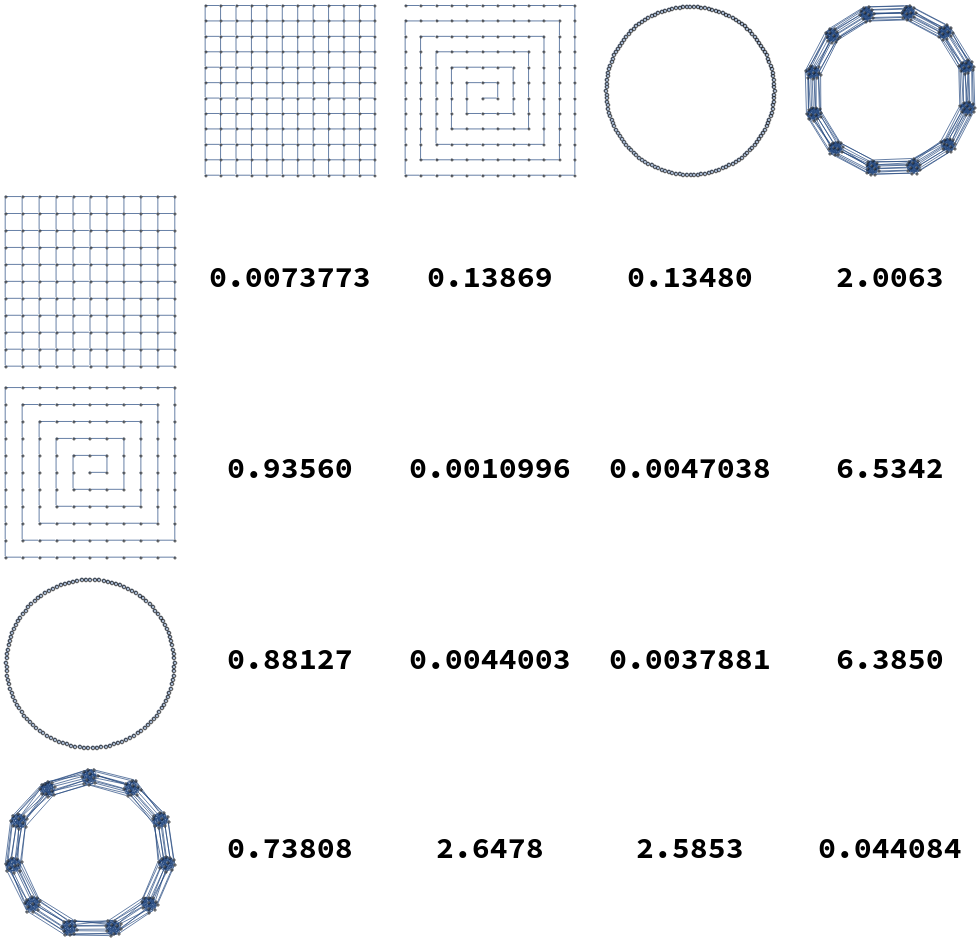} \hfil \null}

    \caption{Distances $D^2(G, H)$ calculated for several pairs of graphs. The top plot shows distances where $G$ and $H$ are both chosen from $\{ \text{Grid}_{13 \times 13}, P_{169}, C_{169}, \text{Ba}_{13} \}$. At bottom, distances are calculated from $G$ chosen in $\{ \text{Grid}_{12 \times 12}, P_{144}, C_{144}, \text{Ba}_{12} \}$ to H chosen in $\{ \text{Grid}_{13 \times 13}, P_{169}, C_{169}, \text{Ba}_{13} \}$. As expected, diagonal entries are smallest. }
    \label{gdist_plots}
\end{figure}
\vspace{-35pt}
\subsubsection{Graph Limits}
\label{subsub:graph_limits}
Here, we provide preliminary evidence that graph distance measures of this type may be used in the definition of a \emph{graph limit} - a graphlike object which is the limit of an infinite sequence of graphs. This idea has been previously explored, most famously by Lov\'asz \cite{lovasz2012large}, whose definition of a graph limit (called a \emph{graphon}) is as follows: Recall the definition of graph cut-distance $D_{\text{cut}}(G, H)$ from Equation \ref{defn:cut_dist}, namely: the cut distance is the maximum discrepancy in sizes of edge-cuts, taken over all possible subsets of vertices, between two graphs on the same vertex-set.
A graphon is then an equivalence class of Cauchy sequences of graphs\footnote{Here we are calling a sequence of graphs ``Cauchy'' if for any $\epsilon > 0$ there is some $N$ such that for all $n, m \geq N$, $D_{\text{cut}}(G_n, G_m) < \epsilon$. }, under the equivalence relation that two sequences $G_1, G_2, \ldots$ and $H_1, H_2, \ldots$ are equivalent if $D_{\text{cut}}(G_i, H_i)$ approaches 0 as $n \rightarrow \infty$.

We propose a similar definition of graph limits, but with our diffusion distance substituted as the distance measure used in the definition of a Cauchy sequence of graphs. Hammond et. al. argue in \cite{hammond2013graph} why their variant of diffusion distance may be a more descriptive distance measure than cut-distance. More specifically, they show that on some classes of graphs, some edge deletions `matter' much more than others: removal of a single edge changes the diffusive properties of the graph significantly. However, the graph-cut distance between the new and old graphs is the same, regardless of which edge has been removed, while the diffusion distance captures this nuance. For graph limits, however, our generalization to \emph{unequal-sized graphs} via $P$ is of course essential. 

We examine several sequences of graphs of increasing size for the required Cauchy behavior (in terms of our distance measure) to justify defining ``graph limits'' in this way. For each of the graph sequences defined in Section \ref{sec:graph_families}, we examine the distance between successive members of the sequence, plotting $D^2(G_n, H_{n+1})$ for each choice of $G$ and $H$.  These sequences of distances are plotted in Figure \ref{fig:g_limits}.

In this figure, we see that generally distance diverges between different graph lineages, and converges for successive members of the same lineage, as $n \rightarrow \infty$. We note the exceptions to this trend:
\begin{enumerate}
    \item The distances between $n$-paths and $n+1$-cycles appear to be converging; this is intuitive, as we would expect that difference between the two spectra due to distortion from the ends of the path graph would decrease in effect as $n \rightarrow \infty$. 
    \item We also show analytically, under similar assumtions, that the distance between successive path graphs also shrinks to 0 (Theorem \ref{thm:pa_lim_lem}). 
\end{enumerate}
We do not show that all similarly-constructed graph sequences display this Cauchy-like behavior. We hope to address this deeper question with one or more modified versions of the objective function (see Section \ref{sec:future}).

In the following theorem (Theorem \ref{thm:pa_lim_lem}), we assume that the optimal value of $t$ approaches some value $\tilde{t}$ as $n \rightarrow \infty$. We have not proven this to be the case, but have observed this behavior for both square grids and path graphs. Lemmas  \ref{lem:path_eigen_limit} and \ref{thm:pa_lim_lem} show a related result for path graphs; we note that the spectrum of the Laplacian (as we define it in this paper) of a path graph of size $n$ is given by 
\begin{equation}
    \begin{aligned}
        \lambda_k = -2 + 2 \cos{\frac{k \pi}{n - 1}} \quad \quad k \in 
        \{0 ... n-1\}.
    \end{aligned}
\end{equation}.

\begin{lemma}
\label{lem:path_eigen_limit}
For any finite $k, t$, we have 
\begin{align}
    \lim_{n \rightarrow \infty} n {\left(
    e^{t(-2 + 2 \cos(\frac{\pi k}{n}))}
    - e^{t(-2 + 2 \cos(\frac{\pi k}{n+1}))}
    \right)}^2 = 0
\end{align}
\end{lemma}
\begin{proof}
Clearly for finite $k,t$
\begin{align}
    \lim_{n \rightarrow \infty} {\left(
    e^{t(-2 + 2 \cos(\frac{\pi k}{n}))}
    - e^{t(-2 + 2 \cos(\frac{\pi k}{n+1}))}
    \right)} = 0
\end{align}
Then, 
\begin{align}
    &\lim_{n \rightarrow \infty} n {\left(
    e^{-2 + 2 \cos(\frac{\pi k}{n})}
    - e^{-2 + 2 \cos(\frac{\pi k}{n+1})}
    \right)}\\
    &\quad \quad \quad = \lim_{n \rightarrow \infty} \frac{ {\left(
    e^{-2 + 2 \cos(\frac{\pi k}{n})}
    - e^{-2 + 2 \cos(\frac{\pi k}{n+1})}
    \right)}}{\frac{1}{n}}
\end{align}
Evaluating this expression requires applying L'H\^{o}pital's rule. Hence, we have:
\begin{align}
    &\lim_{n \rightarrow \infty} \frac{ {\left(
    e^{-2 + 2 \cos(\frac{\pi k}{n})}
    - e^{-2 + 2 \cos(\frac{\pi k}{n+1})}
    \right)}}{\frac{1}{n}} \\
    &\quad \quad = \lim_{n \rightarrow \infty} \frac{
    2 \pi  k t \left(\frac{\sin \left(\frac{\pi  k}{n}\right) e^{2 t \left(\cos \left(\frac{\pi  k}{n}\right)-1\right)}}{n^2}-\frac{\sin \left(\frac{\pi  k}{n+1}\right) e^{2 t \left(\cos \left(\frac{\pi
    k}{n+1}\right)-1\right)}}{(n+1)^2}\right)
    }{\frac{-1}{n^2}} \\
    &\quad \quad = 2 \pi  k t \lim_{n \rightarrow \infty}  \left(\frac{n^2 \sin \left(\frac{\pi  k}{n+1}\right) e^{2 t \left(\cos \left(\frac{\pi  k}{n+1}\right)-1\right)}}{(n+1)^2}-\sin \left(\frac{\pi  k}{n}\right) e^{2 t \left(\cos
   \left(\frac{\pi  k}{n}\right)-1\right)}\right).
   \end{align}
   Since both of the limits
   \begin{align}
   \lim_{n \rightarrow \infty}  \left(\frac{n^2 \sin \left(\frac{\pi  k}{n+1}\right) e^{2 t \left(\cos \left(\frac{\pi  k}{n+1}\right)-1\right)}}{(n+1)^2} \right) \\
   \intertext{and} \\
   \lim_{n \rightarrow \infty} \left(-\sin \left(\frac{\pi  k}{n}\right) e^{2 t \left(\cos
   \left(\frac{\pi  k}{n}\right)-1\right)}\right)
   \end{align}
   exist (and are 0),
   \begin{align}
   & 2 \pi  k t \lim_{n \rightarrow \infty}  \left(\frac{n^2 \sin \left(\frac{\pi  k}{n+1}\right) e^{2 t \left(\cos \left(\frac{\pi  k}{n+1}\right)-1\right)}}{(n+1)^2}-\sin \left(\frac{\pi  k}{n}\right) e^{2 t \left(\cos
   \left(\frac{\pi  k}{n}\right)-1\right)}\right) =0
   \intertext{and therefore}
    & \lim_{n \rightarrow \infty} n {\left(
    e^{t(-2 + 2 \cos(\frac{\pi k}{n}))}
    - e^{t(-2 + 2 \cos(\frac{\pi k}{n+1}))}
    \right)}^2 = 0
\end{align}

\end{proof}

\begin{theorem}
\label{thm:pa_lim_lem} 
If $\lim_{n \rightarrow \infty} \arg \sup_t D^2\left( \left. \text{\emph{Pa}}_n, \text{\emph{Pa}}_{n+1} \right| t \right) $ exists, then:
\begin{align}
    \lim_{n \rightarrow \infty} D^2 \left( \text{Pa}_n, \text{Pa}_{n+1} \right) = 0.
\end{align}
\end{theorem}
\begin{proof}
Assume that $\lim_{n \rightarrow \infty} \arg \sup_t D^2\left( \left. \text{Pa}_n, \text{Pa}_{n+1} \right| t \right)  = \tilde{t}$. Then, we must have
\begin{align}
    \lim_{n \rightarrow \infty} D^2 \left( \text{Pa}_n, \text{Pa}_{n+1} \right) &\leq \lim_{n \rightarrow \infty} D^2 \left( \left. \text{Pa}_n, \text{Pa}_{n+1} \right| \tilde{t} \right)
\end{align}
Hence, it remains only to prove that
\begin{align}
    \lim_{n \rightarrow \infty} D^2 \left( \left. \text{Pa}_n, \text{Pa}_{n+1} \right| t \right) = 0
\end{align}
for any finite $t$ (which will then include $\tilde{t}$). 
First, for any particular $( n + 1) \times n$ subpermutation matrix $S$, note that

\begin{align}
    D^2 \left( \left. \text{Pa}_n, \text{Pa}_{n+1} \right| t \right) &= \inf_{\alpha > 0} \inf_{P | \mathcal{C}(P)} D^2 \left( \left. \text{Pa}_n, \text{Pa}_{n+1} \right| t, P, \alpha \right) \\ 
    &\leq D^2 \left( \left. \text{Pa}_n, \text{Pa}_{n+1} \right| t, \alpha=1 , S\right)
\end{align}
If at each $n$ we select $S$ to be the subpermutation $S = \left[ \begin{array}{c}
     I \\
     0 
\end{array}\right]$, then the objective function simplifies to: 
\begin{align}
    & D^2\left(\text{Pa}_n, \text{Pa}_{n+1} \right| t, P = S, \alpha=1) \\
    &\quad = {\left| \left|  S e^{c \Lambda_{\text{Pa}_n}} - e^{c \Lambda_{\text{Pa}_{n+1}}} S \right| \right|}_F^2 \\
    &\quad = \sum_{k=0}^{n-1} {\left(
    e^{c(-2 + 2 \cos(\frac{\pi k}{n}))}
    - e^{c(-2 + 2 \cos(\frac{\pi k}{n+1}))}
    \right)}^2 \\
    &\quad \leq  \max_{0 \leq k \leq n-1} n {\left(
    e^{c(-2 + 2 \cos(\frac{\pi k}{n}))}
    - e^{c(-2 + 2 \cos(\frac{\pi k}{n+1}))}
    \right)}^2
\end{align}
By Lemma \ref{lem:path_eigen_limit}, for any finite $k, t$, we have 
\begin{align}
\lim_{n \rightarrow \infty} n {\left(
    e^{t(-2 + 2 \cos(\frac{\pi k}{n}))}
    - e^{t(-2 + 2 \cos(\frac{\pi k}{n+1}))}
    \right)}^2 = 0
\end{align}
So for any $\epsilon > 0$, $\exists N$ such that when $n \geq N$, for any $c, k$,
\begin{align}
    n {\left(
    e^{c(-2 + 2 \cos(\frac{\pi k}{n}))}
    - e^{c(-2 + 2 \cos(\frac{\pi k}{n+1}))}
    \right)}^2 < \epsilon
\end{align}
But then 
\begin{align}
\sum_{k=0}^{n-1} {\left(
    e^{c(-2 + 2 \cos(\frac{\pi k}{n}))}
    - e^{c(-2 + 2 \cos(\frac{\pi k}{n+1}))}
    \right)}^2 < \epsilon
\end{align}
as required. Thus, the Cauchy condition is satisfied for the lineage of path graphs $\text{Pa}_n$
\end{proof}

\begin{figure}
    \centering
    \includegraphics[width=\linewidth]{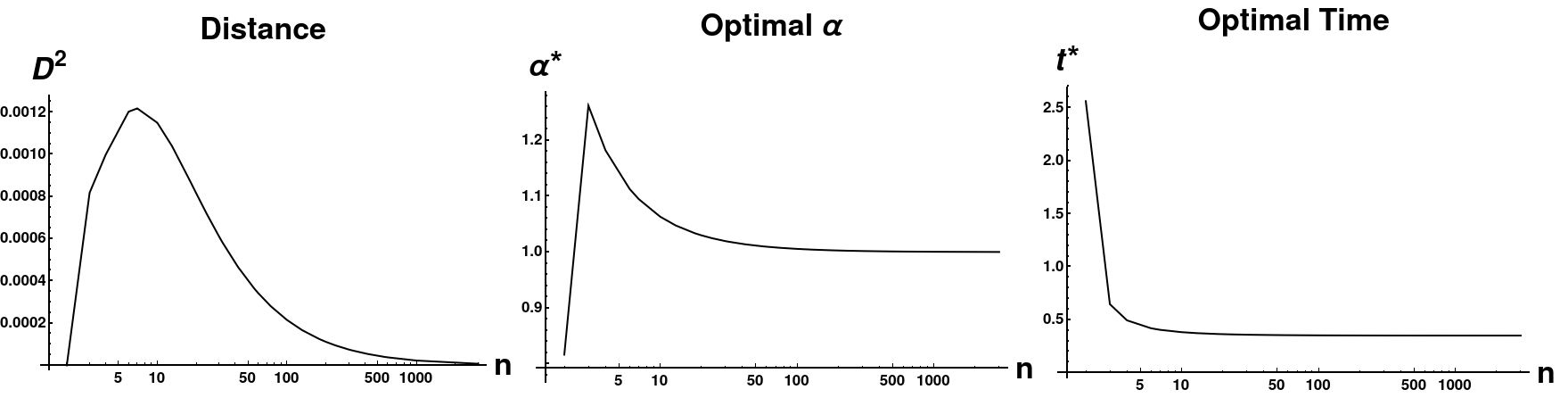}
    \caption{Limiting behavior of $D$ and two parameters as path graph size approaches infinity. All distances were calculated between $\text{Path}_{n}$ and $\text{Path}_{n+1}$. We plot the value of the objective function, as well as the optimal values of $\alpha$ and $t$, as $n \rightarrow \infty$. Optimal $\alpha$ rapidly approach 1 and the optimal distance tends to 0. Additionally, the optimal $t$ value approaches a constant $(t \approx .316345)$, providing experimental validation of the assumption we make in proving Theorem \ref{thm:pa_lim_lem}.}
    \label{fig:path_lim}
\end{figure}

Given a graph lineage which consists of levelwise box products between two lineages, it seems natural to use our upper bound on successive distances between graph box products to prove convergence of the sequence of products. As an example, the lineage consisting of square grids is the levelwise box product of the lineage of path graphs with itself. However, in this we see that this bound may not be very tight. Applying Equation \eqref{eqn:graph_product_special_case} from Theorem \ref{thm:sq_grid_lim_lem}, we have (for any $t_c, \alpha_c$):
\begin{align}
D \left( \left. \text{Sq}_n, \text{Sq}_{n+1} \right. \right) &\leq D \left( \left. \text{Sq}_n, \text{Sq}_{n+1} \right|  t_c, \alpha_c \right)\\ 
&\leq D \left( \text{Pa}_{n+1}, \text{Pa}_{n+1} \left. \right| t_c, \alpha_c \right) \left( {\left| \left| e^{\frac{t_c}{a_c} L(\text{Pa}_n)} \right| \right|}_F \right. \nonumber \\
& \quad \quad \quad \quad + \left. {\left| \left| e^{t_c a_c L(\text{Pa}_{n+1})} \right| \right|}_F \right)  \nonumber 
\end{align}
As we can see in Figure \ref{fig:sq_dist_vs_bound}, the right side of this inequality seems to be tending to a nonzero value as $n \rightarrow \infty$, whereas the actual distance (calculated by our optimization procedure) appears to be tending to 0. 

\begin{figure}
    \centering
    \includegraphics[width=.75\linewidth]{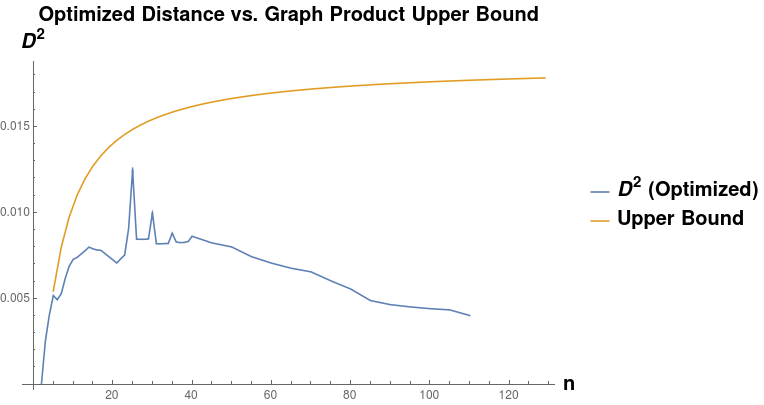}
    \caption{Comparison of the distance $D(\text{Sq}_n, \text{Sq}_{n+1})$ as a function of $n$, to the upper bound calculated as the optimum of distance between $\text{Pa}_n$ and $\text{Pa}_{n+1}$. We see that the upper found converges to some constant $D \approx 0.01782$, whereas the actual distance appears to be converging to 0 as $n \rightarrow \infty$.}
    \label{fig:sq_dist_vs_bound}
\end{figure}

\begin{figure}
    \centering
    \includegraphics[width=\linewidth]{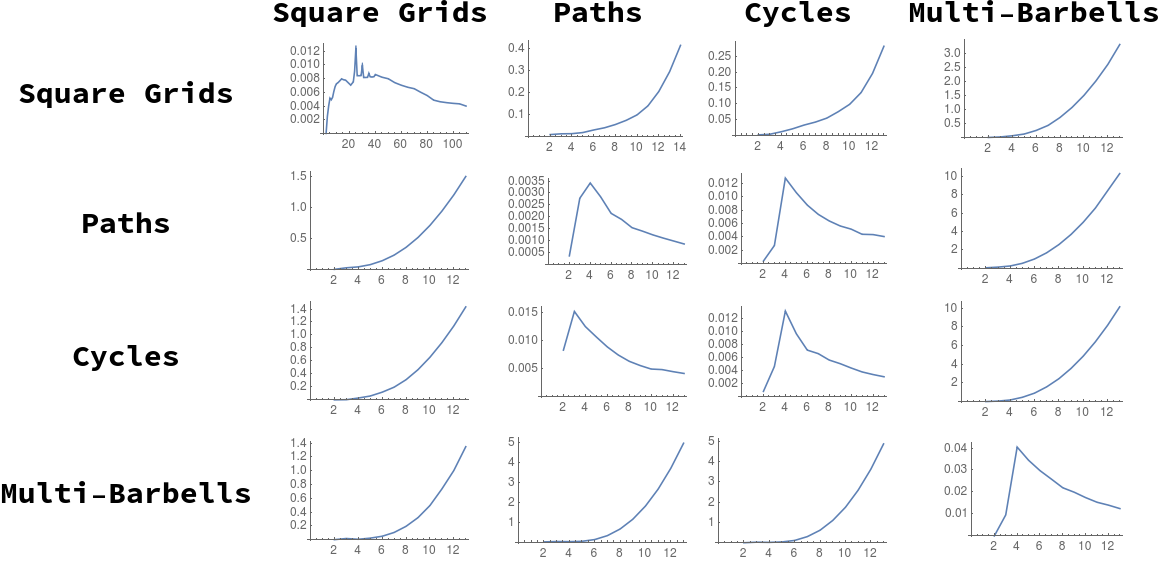}
    \caption{Cauchy-like behavior of graph distance as a function of sequence index, $n$. The distance between successive square grids and all other graph sequences appears to diverge (the same behavior is seen for k-barbells). Notably, the distance between $\text{Grid}_{n \times n}$ and  $\text{Grid}_{(n+1) \times (n+1)}$ does not appear to converge, until much higher values of $n$ $(n > 100)$ than the other convergent series. This may be because the distances calculated are an upper bound, and may be converging more slowly than the `true' optima.} 
    \label{fig:g_limits}
\end{figure}

\section{Regularized Distance}
\label{sec:future}
We can add a regularization term to the graph diffusion distance, as follows: define 
\begin{align}
    D_\text{reg}(G_1, G_2) &= \sup_t \inf_{P|\mathcal{C}(P)} \inf_{\alpha > 0} \left\{ {\left| \left|  P e^{\frac{t}{\alpha} L_1} - e^{t \alpha  L_2 } P \right| \right|}_F + {\left| \left|  e^{\frac{t}{\alpha} L_1} - e^{t L_1 } \right| \right|}_F  \right. \nonumber \\
    &\quad \quad \quad + \left. {\left| \left|  e^{t L_2} P - e^{t \alpha  L_2 } P \right| \right|}_F \right\}
\end{align}
We can show analytically that this distance satisfies the triangle inequality:
\begin{theorem}
$D_\text{reg}$ satisfies the triangle inequality.
\end{theorem}
\begin{proof}
For graphs $G_1, G_2, G_3$ and Laplacians $L_1, L_2, L_3$, for any fixed $t \geq 0$, we have:
\begin{align}
    D_\text{reg}(G_1, G_3 | t) &= \inf_{P|\mathcal{C}(P)} \inf_{\alpha > 0} \left\{  {\left| \left|  P e^{\frac{t}{\alpha} L_1} - e^{t \alpha  L_3 } P \right| \right|}_F + {\left| \left|  e^{\frac{t}{\alpha} L_1} - e^{t L_1 } \right| \right|}_F  \right. \nonumber \\*
    &\quad \quad \quad + \left. {\left| \left|  e^{t L_3} P - e^{t \alpha  L_3 } P \right| \right|}_F \right\} \\
    &\leq D_\text{reg}(G_1, G_3 | t, \alpha=1) \\
    &= \inf_{P|\mathcal{C}(P)} \left\{  {\left| \left|  P e^{t L_1} - e^{t L_3 } P \right| \right|}_F + {\left| \left|  e^{t L_1} - e^{t L_1 } \right| \right|}_F  \right. \nonumber \\
    &\quad \quad \quad + \left. {\left| \left|   e^{t L_3} P - e^{t   L_3 } P \right| \right|}_F \right\} \\
    &= \inf_{P|\mathcal{C}(P)} {\left| \left|  P e^{t L_1} - e^{t L_3 } P \right| \right|}_F
\end{align}
Suppose that 
\begin{align}
    \alpha_{32}, P_{32} &= \arg \inf_{a > 0} \inf_{P | \mathcal{C}(P)}  \left\{ {\left| \left|  P e^{\frac{t}{\alpha} L_2} - e^{t \alpha  L_3 } P \right| \right|}_F + {\left| \left|  e^{\frac{t}{\alpha} L_2} - e^{t L_2 } \right| \right|}_F  \right. \nonumber \\
    &\quad \quad \quad + \left. {\left| \left|  e^{t L_3} P - e^{t \alpha  L_3 } P \right| \right|}_F \right\}\\
    \alpha_{21}, P_{21} &= \arg \inf_{a > 0} \inf_{P | \mathcal{C}(P)}  \left\{ {\left| \left|  P e^{\frac{t}{\alpha} L_1} - e^{t \alpha  L_2 } P \right| \right|}_F + {\left| \left|  e^{\frac{t}{\alpha} L_1} - e^{t L_1 } \right| \right|}_F  \right. \nonumber \\
    &\quad \quad \quad + \left. {\left| \left|  e^{t L_2} P - e^{t \alpha  L_2 } P \right| \right|}_F \right\}\\  
\end{align}
Then, 
\begin{align*}
    \inf_{P|\mathcal{C}(P)} {\left| \left|  P e^{t L_1} - e^{t L_3 } P \right| \right|}_F  &\leq
    {\left| \left|  P_{32} P_{21} e^{t L_1} - e^{t L_3 } P_{32} P_{21} \right| \right|}_F  \\
    \inf_{P|\mathcal{C}(P)} {\left| \left|  P e^{t L_1} - e^{t L_3 } P \right| \right|}_F &\leq 
    {\left| \left|  P_{32} P_{21} e^{t L_1} - P_{32} P_{21} e^{\frac{t}{\alpha_{21}} L_1} \right. \right. }  + P_{32} P_{21} e^{\frac{t}{\alpha_{21}} L_1} \\*
    & \quad - P_{32} e^{t \alpha_{21} L_2 } P_{21} + P_{32} e^{t \alpha_{21} L_2 } P_{21} - P_{32} e^{t L_2 } P_{21}  \\*
    & \quad + P_{32} e^{t L_2 } P_{21} - P_{32} e^{\frac{t}{\alpha_{32}} L_2 } P_{21}  + P_{32} e^{\frac{t}{\alpha_{32}} L_2 } P_{21}  \\*
    & \quad - e^{t \alpha_{32} L_3 } P_{32} P_{21} {\left. \left. + e^{t \alpha_{32} L_3 } P_{32} P_{21} - e^{t L_3 } P_{32} P_{21} \right| \right|}_F \\
    &\leq 
    {\left| \left|  P_{32} P_{21} e^{t L_1} - P_{32} P_{21} e^{\frac{t}{\alpha_{21}} L_1}\right| \right|}_F  \\*
    & \quad  + {\left| \left| P_{32} P_{21} e^{\frac{t}{\alpha_{21}} L_1} - P_{32} e^{t \alpha_{21} L_2 } P_{21} \right| \right|}_F \\*
    & \quad  + {\left| \left| P_{32} e^{t \alpha_{21} L_2 } P_{21} -P_{32} e^{t L_2 } P_{21} \right| \right|}_F  \\*
    & \quad + {\left| \left| P_{32} e^{t L_2 } P_{21} - P_{32} e^{\frac{t}{\alpha_{32}} L_2 } P_{21} \right| \right|}_F \\*
    & \quad  + {\left| \left| P_{32} e^{\frac{t}{\alpha_{32}} L_2 } P_{21} - e^{t \alpha_{32} L_3 } P_{32} P_{21} \right| \right|}_F  \\*
    & \quad + {\left| \left| e^{t \alpha_{32} L_3 } P_{32} P_{21} - e^{t L_3 } P_{32} P_{21} \right| \right|}_F \\
    \intertext{by Lemma \ref{lem:p_lem_1},}
    \inf_{P|\mathcal{C}(P)} {\left| \left|  P e^{t L_1} - e^{t L_3 } P \right| \right|}_F &\leq 
    {\left| \left|   e^{t L_1} -  e^{\frac{t}{\alpha_{21}} L_1}\right| \right|}_F  + {\left| \left| P_{21} e^{\frac{t}{\alpha_{21}} L_1} - e^{t \alpha_{21} L_2 } P_{21} \right| \right|}_F \\
    & \quad  + {\left| \left|  e^{t \alpha_{21} L_2 } P_{21} - e^{t L_2 } P_{21} \right| \right|}_F \\
    &+ {\left| \left| e^{t L_2 } -  e^{\frac{t}{\alpha_{32}} L_2 } \right| \right|}_F  + {\left| \left| P_{32} e^{\frac{t}{\alpha_{32}} L_2 }  - e^{t \alpha_{32} L_3 } P_{32} \right| \right|}_F  \\
    & \quad + {\left| \left| e^{t \alpha_{32} L_3 } P_{32} - e^{t L_3 } P_{32} \right| \right|}_F \\
    &= D_\text{reg}(G_1, G_2 | t = c) + D_\text{reg}(G_2, G_3 | t = c)
\end{align*}
Since this is true for any fixed $t$, let 
\begin{align*}
    t^* = \arg \sup_t D_\text{reg}(G_1, G_3 | t)
\end{align*}.
Then 
\begin{align*}
    D_\text{reg}(G_1, G_3) &= \sup_c D_\text{reg}(G_1, G_3 | t) \\
    &= D_\text{reg}(G_1, G_3 | t^*) \\
    & \leq D_\text{reg}(G_1, G_2 | t^*) + D_\text{reg}(G_2, G_3 | t = t^*) \\
    & \leq \sup_{t_{21}} D_\text{reg}(G_1, G_2 | t_{21}) + \sup_{t_{32}} D_\text{reg}(G_2, G_3 | t_{32}) \\
    &= D_\text{reg}(G_1, G_2) + D_\text{reg}(G_2, G_3)
\end{align*}

\end{proof}

We can construct a similar regularized version of the linear objective function:
\begin{align*}
    \tilde{D}_\text{reg}(G_1, G_2) &= {\left| \left| \frac{1}{\alpha} P L_1 - \alpha L_2 P \right| \right|} + {\left| \left| \frac{1}{\alpha} L_1 - L_1 \right| \right|} + {\left| \left| \frac{\null}{\null} P L_2 - \alpha L_2 P \right| \right|}
\end{align*}

The term ``regularized'' here refers to the fact that the additional terms included in $D_\text{reg}$ and $\tilde{D}_\text{reg}$ penalize $\alpha$ distorting the respective Laplacians far from their original values. In practice, many of the theoretical guarantees provided earlier in this manuscript may not apply to optimization of the augmented objective function. Hence, a major area of future work will be modification of our optimization procedure to compute this form of distance.

\section{Applications and Future Work}
\label{sec:futurework}
We briefly discuss possible applications of both our distance metric and our procedure for calculating the relevant minima. 

\subsection{Algebraic Multigrid}
\label{subsec:appl_amg}
The need for prolongation/restriction operators arises naturally in the Algebraic MultiGrid (AMG) context, where a hierarchy of progressively coarser meshes are constructed, with the goal of speeding convergence of a model with local update (``smoothing'') rules. A model with modes of behavior at wavelengths which are much larger than the neighborhood of one update will take many update steps to converge. Thus, the goal in AMG is to iteratively construct a series of coarsened meshes, so that update steps at the coarser scales can address coarser modes of behavior. A fine-scale model state is translated into a coarse-scale state via a ``restriction'' operator. After a coarse-scale smoothing step, the new coarse state is translated back to the fine-scale by ``prolonging'' it. Our procedure for calculating $P$ could be incorporated as a preprocessing step, in the case where the series of meshes are known in advance; otherwise, the $P$ from the previous round of coarsening could be used as the initial conditions to a modified version of our solver. In either case, the matrix $P$ is a natrual choice of prolongation/restriction operator for this type of coarsening scheme, since it optimally transforms the Laplacian of one graph into another. 

\subsection{Graph Limits}
\label{subsec:appl_graph_limits}
In this work we briefly introduce a new definition of graph limits based on the diffusion distance, which raises several natural questions: What does the ``limit'' of a sequence of graphs under diffusion distance look like? Are there pairs of sequences that converge to the same such object, as in the example of path graphs and cycle graphs? Can we separate graph sequences into equivalence classes based on which of these they converge to? We hope to address these questions in future work. 

\subsection{Graph Convolutional Networks}
\label{subsec:appl_gcns}
Graph convolutional networks (GCNs) are a variant of the \emph{convolutional neural networks} (CNNs) widely used in machine vision. In the same way that CNNs learn a set of trained image filters and apply them across multiple spatial locations in an image, GCNs learn a set of filters which are applied to local neighborhoods of a graph. One implementation of GCNs due to Kipf and Welling \cite{kipf2016semi} uses a Chebyshev polynomial of the Laplacian matrix as an approximation of the graph fourier transform, demonstrating comparable results to the full transform but far fewer multiplication operation needed. However, construction of pooling operators for GCNs is still an area of open research. Since our $P$ is a restriction operator that preserves information about the Laplacian, it is natural to use it as a pooling operator in this type of model. 

\subsection{Graph Clustering}
\label{subsec:appl_graph_clustering}
We can also use the diffusion distance and its variants to compare graphs or neighborhoods of graphs for structural similarity, independent of graph size. This is similar to the approach of \cite{gold1995new} for comparing point clouds in 2D and 3D, in the sense that both approaches optimize an objective function based to a matching between elements of the two graphs. This type of similarity measure may then be used to convert a dataset of graphs to a distance-to-cluster-centers representation, or for any other of the typical methods used in machine learning for converting sets of pairwise distances into fixed-length feature vectors ($k$-medoids, kernel methods, multidimensional scaling, etc.). In this setting, our distance measure has an additional benefit: since computing it yields an explicit projection operator between the nodes of the graphs, we may use the set of $P$ we compute to project signals (e.g. labels on the vertices of each graph in the dataset) to a common space. 

\section{Conclusion}
In this work, we present a novel generalization of graph diffusion distance which allows for comparison of graphs of inequal size. We consider several variants of this distance measure to account for sparse maps between the two graphs, and for maps between the two graphs which are optimal given a fixed time-dilation factor $\alpha$. We prove several important theory properties of distances in this family of measures, including triangle inequalities in some cases and Cauchy-like behavior of some graph sequences. We present a new procedure for optimizing the objective function defined by our distance measure, prove the correctness of this procedure, and demonstrate its efficiency in comparison to univariate search over the dilation parameter, $\alpha$. Numerical experiments suggest that this dissimilarity score satisfies the triangle inequality up to some constant $\rho \approx 2.1$. We demonstrate that this measure of graph distance may be used to compare graph lineages (families of exponentially-growing graphs with shared structure), and additionally that certain lineages display Cauchy-sequence like behavior as the graph size approaches infinity. We suggest several possible applications of our distance measure to scientific problems in the contexts of  pattern matching and machine learning.

\section{Acknowledgements}
\ifanon
\else
This work was supported by U.S. National Science Foundation NRT Award number 1633631, 
U.S. National Institute of Aging grant AG059602,
Human Frontiers Science Program grant HFSP - RGP0023/ 2018,
U.S. National Institutes for Health grant R01HD073179,
USAF/DARPA FA8750-14-C-0011,
and by the Leverhulme Trust and
and the hospitality of the Sainsbury Laboratory at Cambridge University, as well as the hospitality of the Center for Non-Linear Studies (CNLS) at Los Alamos National Laboratory.
\fi

\bibliographystyle{spmpsci}
\bibliography{references}

\section{Appendices}
\subsection{Spectral Lower Bound}
\label{appndx:slb}
In Theorem \ref{thm:LAP_bound} we derived an upper bound on the graph distance $\tilde{D}(G_1, G_2)$, by constraining the variable $P$ to be not only orthogonal, but also $ P = U_2 M U_1^T$ where M is the solution (i.e. ``matching'', in the terminology of that section) to a Linear Assignment problem with costs given by a function of the eigenvalues of $L(G_1)$ and $L(G_2)$. In this appendix we derive a similar lower bound on the distance. 

As in Theorem \ref{thm:LAP_bound}, let $G_1$ and 
$G_2$ be undirected graphs with Laplacians $L_1 = L(G_1)$ and $L_2 = L(G_2)$, and let $\alpha > 0$ be constant. By Equation \eqref{eqn:simple_d_tilde}, we have 

\begin{equation}
    \begin{aligned}
        \tilde{D}^2(G_1,G_2) &= \inf_{\alpha > 0} \inf_{P^T P = I} \left(
        \sum\limits_{i=1}^{n_2}\sum\limits_{j=1}^{n_1}
        p_{ij}^2 {\left( 
        \frac{1}{\alpha} {\lambda^{(1)}_j} - 
        \alpha {\lambda^{(2)}_i} \right)}^2
        \right).
    \end{aligned}
\end{equation}
Previously, we derived an upper bound on $\tilde{D}$ by constraining $P$ to be not only orthogonal, but related to a constrained matching problem between the two lists of eigenvalues: 

    \begin{equation}
        \label{eqn:upper_bound_2_defn}
        \begin{array}{ll@{}ll}
        \tilde{D}^2(G_1,G_2) \leq \inf_{\alpha>0} \inf_{M} & {\left| \left| \frac{1}{\alpha} M \Lambda_1 - \alpha \Lambda_2 M \right| \right|}^2_F &\\
        \text{subject to}& \displaystyle\sum\limits_{i=1}^{n_2}  m_{ij} \leq 1,  &j=1 \ldots n_1\\
               & \displaystyle\sum\limits_{j=1}^{n_1}  m_{ij} \leq 1,  &i=1 \ldots n_2\\
         & m_{ij} \geq 0 & i = 1 \ldots n_2, j = 1 \ldots n_1, \\
        \end{array}
    \end{equation}
where $\Lambda_1$ and $\Lambda_2$ are diagonal matrices of the eigenvalues of $L_1$ and $L_2$ respectively. In that proof, we used the explicit map $\tilde{P} = U_2^T P U_1$ as a change of basis; we then converted the constraints on $P$ into equivalent constraints on $\tilde{P}$, and imposed additional constraints so that the resulting optimization (a linear assignment problem) was an upper bound. We show in this section that a more loosely constrained assignment problem is a lower bound on $\tilde{D}^2$. We do this by computing the same mapping  $\tilde{P} = U_2^T P U_1$  and then dropping some of the constraints on $\tilde{P}$ (which is equivalent to dropping constraints on $P$, yielding a lower bound). The constraint $P^T P = I$ is the conjunction of $n_1 ^ 2$ constraints on the column vectors of $P$: if $\mathbf{p}_i$ is the $i$th column of $P$, then $P^T P = I$ is equivalent to:
\begin{align}
    \mathbf{p}_i \cdot \mathbf{p}_i = 1 & \qquad & \forall i=1 \ldots n_1 \label{eqn:p_const_l1}\\
    \mathbf{p}_i \cdot \mathbf{p}_i = 0 & \qquad & \forall i = 1 \ldots n_1, j = 1 \ldots i-1, i+1 \ldots n_1, \label{eqn:p_const_l2} 
\end{align}
If we discard the constraints in Equation \eqref{eqn:p_const_l2}, we are left with the constraint that every column of $p$ must have unit norm. 

Construct the ``spectral lower bound matching'' matrix $P^{(\text{SLB})}$ as follows: 

\begin{equation}
  P^{(\text{SLB})}_{i,j} =
  \begin{cases}
    1 & \text{if i = } \arg \min_k {\left( \frac{1}{\alpha}  \lambda^{(1)}_j - \alpha \lambda^{(k)}_k \right)}^2 \\
    0 & \text{otherwise.}
  \end{cases}
\end{equation}
For any $\alpha$, this matrix is the solution to a matching problem (less constrained than the original optimization over all $P$) where each $\lambda^{(1)}_j$ is assigned to the closest $\lambda^{(2)}_i$, allowing collisions. It clearly satisfies the constraints in Equation \eqref{eqn:p_const_l1}, but may violate those in Equation \eqref{eqn:p_const_l2}. Thus, we have 

\begin{equation}
    \begin{aligned}
        \tilde{D}^2(G_1,G_2) &= \inf_{\alpha > 0} \inf_{P^T P = I} \left(
        \sum\limits_{i=1}^{n_2}\sum\limits_{j=1}^{n_1}
        p_{ij}^2 {\left( 
        \frac{1}{\alpha} {\lambda^{(1)}_j} - 
        \alpha {\lambda^{(2)}_i} \right)}^2
        \right). \\
        &\geq \tilde{D}^2\left(G_1,G_2 \left| P^{(\text{SLB})} \right.\right)
    \end{aligned}
\end{equation}

Various algorithms exist to rapidly find the member of a set of points which is closest to some reference point (for example, KD-Trees \cite{bentley1975multidimensional}). For any $\alpha$, the spectral lower bound can be calculated by an outer loop over alpha and an inner loop which applies one of these methods. We do not consider joint optimization of the lower bound over $P$ and $\alpha$ in this work. 

\end{document}